\documentclass[openacc]{rsproca_new}%%%% Royal Society Proceedings A template

\usepackage[utf8]{inputenc}

\usepackage{amsmath,amssymb,amsfonts,amsthm}

\newcommand*\E[1]{\mathbb{E}\left[#1\right]}
\newcommand*\Ep[2]{\mathbb{E}_{#1}\left[#2\right]}

\newcommand{\vx}{\mathbf{x}}
\newcommand{\vy}{\mathbf{y}}

\newcommand*\lrn[1]{\left\|#1\right\|}

\newtheorem{theorem}{Theorem}
\newtheorem{corollary}{Corollary}

\newtheorem{remark}{Remark}

\usepackage{amsmath,amsfonts,bm}

\def\eqref#1{equation~\ref{#1}}
\def\1{\bm{1}}

\def\vx{{\bm{x}}}
\def\vy{{\bm{y}}}

\DeclareMathAlphabet{\mathsfit}{\encodingdefault}{\sfdefault}{m}{sl}
\SetMathAlphabet{\mathsfit}{bold}{\encodingdefault}{\sfdefault}{bx}{n}

\newcommand{\R}{\mathbb{R}}

\usepackage{subcaption}
\usepackage{url}
\usepackage{booktabs}

\DeclareCaptionFont{rscaption}{\fontfamily{\sfdefault}\fontseries{m}\fontshape{n}\fontsize{8.5}{10}\selectfont}
\DeclareCaptionFont{rscaptionhead}{\fontfamily{\sfdefault}\fontseries{b}\fontshape{n}\fontsize{8.5}{10}\selectfont}
\titlehead{Research}

\begin{document}

%%%% Article title
\title{UQ-SHRED: uncertainty quantification of shallow recurrent decoder networks for sparse sensing via engression}

%%%% Author details
\author{Mars Liyao Gao$^{1,\ast}$, Yuxuan Bao$^{2,\ast}$, Amy S. Rude$^{2}$, Xingyue Zhang$^{3}$, Xinwei Shen$^{4}$ and J. Nathan Kutz$^{1,5}$}

%%%% Author addresses
\address{$^{1}$Department of Statistics, Stanford University, Stanford, CA, USA\\
$^{2}$Department of Applied Mathematics, University of Washington, Seattle, WA, USA\\
$^{3}$School of Environmental and Forest Sciences, University of Washington, Seattle, WA, USA\\
$^{4}$Department of Statistics, University of Washington, Seattle, WA, USA\\
$^{5}$Autodesk Research, London, UK\\[4pt]
$^{\ast}$These authors contributed equally.}

%%%% Subject entries -- TODO: please confirm/edit (choose from the Royal Society subject list)
\subject{applied mathematics, statistics, computer modelling and simulation}

%%%% Keyword entries -- TODO: please confirm/edit

%%%% Corresponding authors and email addresses -- TODO: please confirm/edit
\corres{Mars Liyao Gao, J. Nathan Kutz\\
\email{marsgao@stanford.edu}\\
\email{kutz@uw.edu}}

%%%% Abstract
\begin{abstract}
Reconstructing high-dimensional spatiotemporal fields from sparse sensor measurements is critical in a wide range of scientific applications. The SHallow REcurrent Decoder (SHRED) architecture is a recent state-of-the-art architecture that reconstructs high-quality spatial domain from hyper-sparse sensor measurement streams. 
An important limitation of SHRED is that in complex, data-scarce, high-frequency, or stochastic systems, portions of the spatiotemporal field must be modeled with valid uncertainty estimation. We introduce \textbf{UQ-SHRED}, a distributional learning framework for sparse sensing problems that provides uncertainty quantification through a neural network-based distributional regression called engression. UQ-SHRED models the uncertainty by learning the predictive distribution of the spatial state conditioned on the sensor history.\absbreak By injecting stochastic noise into sensor inputs and training with an energy score loss, UQ-SHRED produces predictive distributions with minimal computational overhead, requiring only noise injection at the input and resampling through a single architecture without retraining or additional network structures. On complicated synthetic and real-life datasets including turbulent flow, atmospheric dynamics, neuroscience and astrophysics, UQ-SHRED provides a distributional approximation with well-calibrated confidence intervals. 
We further conduct ablation studies to understand how each model setting affects the quality of the UQ-SHRED performance, and its validity on uncertainty quantification over a set of different experimental setups.
\end{abstract}

\rsbreak

\section{Introduction}
\label{sec:intro}

Reconstructing high-dimensional spatiotemporal fields from limited and sparse sensor measurements is a fundamental problem across the physical, biological and engineering sciences.
In each of these settings, the reconstruction problem is highly underdetermined, as the spatial dimensionality vastly exceeds the number of available sensor measurements.  
Data-driven approaches to sparse sensing have advanced considerably in recent years~\cite{manohar2018data}, with deep learning architectures demonstrating strong empirical performance on high-dimensional reconstruction tasks~\cite{erichson2025flex,erichson2020shallow,yu2019flowfield}. 
The SHallow REcurrent Decoder (SHRED) network~\cite{williams2024sensing} takes a distinct approach by leveraging the temporal history of a hyper-sparse number of sensor measurements into a reduced-order latent space, which is grounded in Takens' delay embedding theorem~\cite{takens2006detecting}. 
SHRED thus enables sparse spatial sensing from as few as three sensors placed anywhere in the dynamics. This hyper-sparse sampling and sensor agnostic placement strategy have been demonstrated on a number of scientific applications such as fluid dynamics~\cite{williams2024sensing}, seismology~\cite{ni2024wavefield}, neuroscience~\cite{rude2025shallow}, plasma physics~\cite{kutz2024shallow}, and remote sensing~\cite{zhang2026sendai}. The SHRED architecture is flexible and can also be modified for model discovery~\cite{gao2026sparse} and data assimilation~\cite{bao2025data}. 
Challenges remain however, especially as point estimations are fundamentally limited for downstream tasks such as risk assessment, anomaly detection, and decision-making under uncertainty. 
Therefore, it is essential to extend the SHRED architecture to quantify the uncertainty with validity in the sensor-to-state mapping, enabling uncertainty-aware analysis for reliable and safety-critical downstream scientific applications. Despite its importance, obtaining reliable uncertainty estimates in this setting is highly nontrivial. 
The difficulties come from the fact that (i) the use of blackbox learning architectures introduces nonlinear mappings, (ii) recurrent units accumulate and amplify uncertainty over time, and (iii) many existing approaches to distributional learning require substantial data and computational resources (e.g. ensembling) while often lacking formal statistical guarantees on the resulting predictive distribution.

In this paper, we introduce UQ-SHRED, a distributional learning framework for sparse sensing that enables valid uncertainty quantification by learning the full conditional distribution of the spatial states given observed sensor measurements. It is achieved by concatenating Gaussian noise vectors to the input and training the model using an energy score loss to directly optimize a proper scoring rule for distributional predictions inspired by Shen and Meinshausen~\cite{shen2023engression}.
Since the noise is injected at the input, it propagates through the temporal unit and shallow decoder, allowing uncertainty to be modeled throughout the full mapping from sparse sensors to spatial states without architectural modification. At inference time, UQ-SHRED generates samples from the conditional predictive distribution by drawing samples of the input noise distribution and propagating these samples through the network. The resulting empirical distribution enables estimation of statistical quantities of interests, including the conditional mean, variance, and prediction intervals. Crucially, UQ-SHRED requires only a single trained network which avoids the computational cost of training multiple separate units for distributional learning. The method further admits theoretical guarantees on the derived conditional distribution and the quantile estimates.

We evaluate our method on five complex real-world datasets including sea-surface temperature data from NOAA~\cite{reynolds2002improved}, isotropic turbulent flow from JHUDB~\cite{li2008public}, neural data from the Allen Institute~\cite{siegle2021survey}, solar activity from NASA Solar Dynamics Observatory~\cite{pesnell2012solar}, and propulsion physics~\cite{bao2026cheap2rich}. 
These datasets provide a broad testbed for uncertainty quantification in sparse recovery, spanning different spatial structures, sensing regimes, and system dynamics. We validate the UQ-SHRED architecture and its robustness across these scientific applications.
Furthermore, we present an ablation study on the SST dataset to investigate the sensitivity of calibration quality to temporal lag, sensor count, noise dimension, Monte Carlo sample size, and training duration. 
The ablation study further illustrates the behavior of the uncertainty estimates produced by UQ-SHRED under these changing, potentially non-ideal hyperparameter settings. 
To summarize, our UQ-SHRED contributions are as follows:

\begin{enumerate}
    \item We develop a distributional learning framework for sparse sensing that provides valid uncertainty quantification. Inspired by engression, UQ-SHRED combines noise injection and energy score minimization for learning the predictive distribution across the temporal unit and shallow decoder of the SHRED architecture, while preserving the computational efficiency of a single-network approach.
    \item We establish the theoretical validity of UQ-SHRED under regularity conditions, showing that the learned conditional distribution supports valid estimation of common statistical quantities of interest. This distinguishes UQ-SHRED from purely heuristic uncertainty estimates and provides theoretical support for its use in uncertainty-aware sparse recovery.
    \item We verify the effectiveness of UQ-SHRED on a wide range of scientific problems including sea-surface temperature, turbulent flows, neural activity, solar activity, and propulsion physics. This demonstrates that UQ-SHRED remains effective across diverse applications with remarkably different physics.
\end{enumerate}

\noindent
The manuscript is organized as follows:  In Sec.~\ref{sec:prelim}, we detail the two critical mathematical structures of UQ-SHRED and highlight some of the current methods being used for UQ in deep learning.  This is followed in Sec.~\ref{sec:method} with the development of the UQ-SHRED architecture along with theoretical estimates related to the methodology.  Section~\ref{sec:experiments} deploys the UQ-SHRED architecture across a diversity of real-world data sets, explicitly demonstrating the quality of reconstructions and the accuracy of the theoretical estimates derived.  An outlook and review of the findings is given in Sec.~\ref{sec:conclusion}.

\section{Preliminaries}
\label{sec:prelim}

Before presenting our UQ-SHRED architecture, the two mathematical architectures required are detailed.  Moreover, we also give a brief review of other existing methods for extracting UQ metrics using deep learning.

\subsection{The SHRED Architecture}
\label{sec:shred_background}

As noted earlier, reconstructing high-dimensional spatiotemporal states from sparse sensor measurements is a core problem in scientific computing and engineering~\cite{kutz2025accelerating}. The SHRED architecture~\cite{williams2024sensing,ye2025pyshred} addresses this by exploiting the temporal structure at sensor locations to infer the full spatial field. The approach draws on Takens' delay embedding theorem~\cite{takens2006detecting}, which guarantees that, under generic conditions, a sufficiently long time-history of a scalar observable provides a diffeomorphic representation of the underlying dynamical attractor. Concretely as it relates to SHRED, given $p$ sensors at fixed spatial locations, a sliding window of $L$ consecutive measurements $\mathbf{S}_{t-L+1:t} \in \mathbb{R}^{L \times p}$ is fed into a multi-layer LSTM~\cite{hochreiter1997long} that produces a latent code $\mathbf{z}_t \in \mathbb{R}^{d_z}$ from its final hidden state. A shallow decoder then maps $\mathbf{z}_t$ to the full spatial field $\hat{\mathbf{u}}_t \in \mathbb{R}^n$ with $n \gg p$.

Two properties of this architecture are particularly relevant for uncertainty quantification. First, the sensor-to-state mapping is inherently underdetermined: $p$ sensor values do not uniquely specify $n$ spatial degrees of freedom, and the LSTM-based structure imposes a finite-dimensional information bottleneck through $d_z$. Any deterministic SHRED model selects a single point estimate from the set of states consistent with the observed sensors, discarding information about the remaining ambiguity. Second, the temporal unit processes the lag window sequentially, meaning its hidden state serves as a compressed summary of the recent dynamical history. Injecting controlled stochasticity into this process---as we propose in Section~\ref{sec:noise_injection}---transforms the deterministic point estimate into a distribution over plausible reconstructions, with the spread reflecting the sensor-to-state ambiguity at each spatiotemporal location.

Several extensions of SHRED have been developed for specific modeling challenges, including latent-space dynamics identification for general purpose scientific discovery~\cite{gao2026sparse}, reduced-order modeling~\cite{tomasetto2025reduced}, and transformer-based temporal units~\cite{yermakov2025t}. DA-SHRED~\cite{bao2025data} addresses the simulation-to-reality gap by updating the latent representations of SHRED using real sensor observations, enabling data assimilation without access to full-field ground-truth data. In this work, we effectively regularize SHRED with the engression framework for reconstruction tasks from sparse observation, though the approach could be compatible with any variant that admits noise injection at the input layer, including multiscale extensions for settings with heterogeneous fields.

\subsection{Engression and the Energy Score}
\label{sec:engression_background}
Engression~\cite{shen2023engression} is a neural network-based distributional learning method that aims to learn the full conditional law $\mathcal{L}(Y\mid X=x)$ rather than only the conditional mean $\E{Y\mid X=x}$. 
Engression models the conditional distribution through a general conditional generative model of the form
\begin{align}
    Y=g_\theta(X,\boldsymbol{\varepsilon}),
\end{align}
where $\boldsymbol{\varepsilon}\sim P_{\boldsymbol{\varepsilon}}$ and $\boldsymbol{\varepsilon}\perp X$. Engression fits a conditional sampler $g_\theta(x,\boldsymbol{\varepsilon})$ which induces a predictive distribution $P_\theta(\cdot\mid x)$. 

For a distribution $P$ on $\R^r$ and observation $z\in\R^r$, let $Z,Z'\overset{i.i.d.}{\sim}P$ be two independent random vectors drawn from $P$. The proper scoring rule used in engression is the Energy Score
\begin{align}
    ES(P,z)=\frac{1}{2}\Ep{P}{\|Z-Z'\|_2}-\Ep{P}{\|Z-z\|_2},
    \qquad Z,Z'\overset{i.i.d.}{\sim}P,
\end{align}
which is a multivariate extension of the continuous ranked probability score (CRPS), and is strictly proper under mild conditions~\cite{szekely2003statistics,szekely2023energy}. 
Plugging in $P=P_\theta(\cdot\mid X)$ yields the population objective: 
\begin{align}
    \label{eqn:eng_energy_score}
    \min_{\theta}\ \E{\|Y-g_\theta(X,\boldsymbol{\varepsilon})\|_2-\frac{1}{2}\|g_\theta(X,\boldsymbol{\varepsilon})-g_\theta(X,\boldsymbol{\varepsilon}')\|_2 },
    \qquad \boldsymbol{\varepsilon},\boldsymbol{\varepsilon}' \overset{i.i.d.}{\sim} P_{\boldsymbol{\varepsilon}},
\end{align}
which enables likelihood-free distribution learning for multivariate $Y$ using only samples from $g_\theta$ by drawing Monte Carlo samples from $g_\theta(x, \boldsymbol{\varepsilon}_j)$ with $\boldsymbol{\varepsilon}_j\overset{i.i.d.}{\sim}P_{\boldsymbol{\varepsilon}}$. 

\subsection{Uncertainty quantification methods in deep learning}

Uncertainty quantification (UQ) for deep learning models is an active area of research with a broad and diverse range of methods for extracting UQ~\cite{he2026survey,wu2021quantifying}. Bayesian neural networks~\cite{zhu2018bayesian,neal2012bayesian,blundell2015weight,tonekaboni2025hdp,mars2024bayesian} place prior distributions over network weights and approximate the posterior via variational inference or sampling, providing a principled but often computationally demanding framework. MC-dropout~\cite{gal2016dropout} reinterprets dropout at test time as approximate Bayesian inference, offering a practical alternative that requires no architectural changes but lacks formal consistency guarantees for the resulting predictive distribution. Deep ensembles~\cite{lakshminarayanan2017simple,guo2017calibration} train multiple independent networks with different random initializations and aggregate their predictions; while empirically effective, this approach multiplies the training and inference cost by the ensemble size and does not optimize a proper scoring rule. Heteroscedastic regression models~\cite{kendall2017uncertainties} parameterize the output variance as a function of the input, but typically assume a fixed distributional form (e.g., Gaussian) that may be misspecified for complex spatiotemporal systems. More recently, generative approaches including conditional normalizing flows~\cite{papamakarios2021normalizing,trippe2018conditional} and diffusion models~\cite{ho2020denoising,song2020score} can represent flexible conditional distributions without parametric assumptions on the output, but normalizing flows require invertible network architectures that can limit expressiveness in higher-dimensional output spaces, and diffusion models incur significant computational cost from the iterative denoising procedure required during both training and inference.

\section{Method: UQ-SHRED}
\label{sec:method}

The deterministic SHRED model (introduced in Section~\ref{sec:shred_background}) returns a single point estimate $\hat{\vy}_t$ for each sensor window $\vx_t$. However, in regimes characterized by stochastic dynamics, high-frequency content, or insufficient data coverage, portions of the spatiotemporal field cannot be reconstructed with high confidence. In such settings, the quantity of interest is not just a best-guess reconstruction but the full conditional distribution $P(\vy_t \mid \vx_t)$, from which one can derive prediction intervals, assess risk, and propagate uncertainty into downstream analyses. This section develops UQ-SHRED, which leverages the engression framework~\ref{sec:engression_background} to learn the predictive distribution of the spatial state conditioned on the sparse sensor history.

\subsection{Architecture}
\label{sec:noise_injection}

\begin{figure}[t]
\centering
\includegraphics[width=\columnwidth]{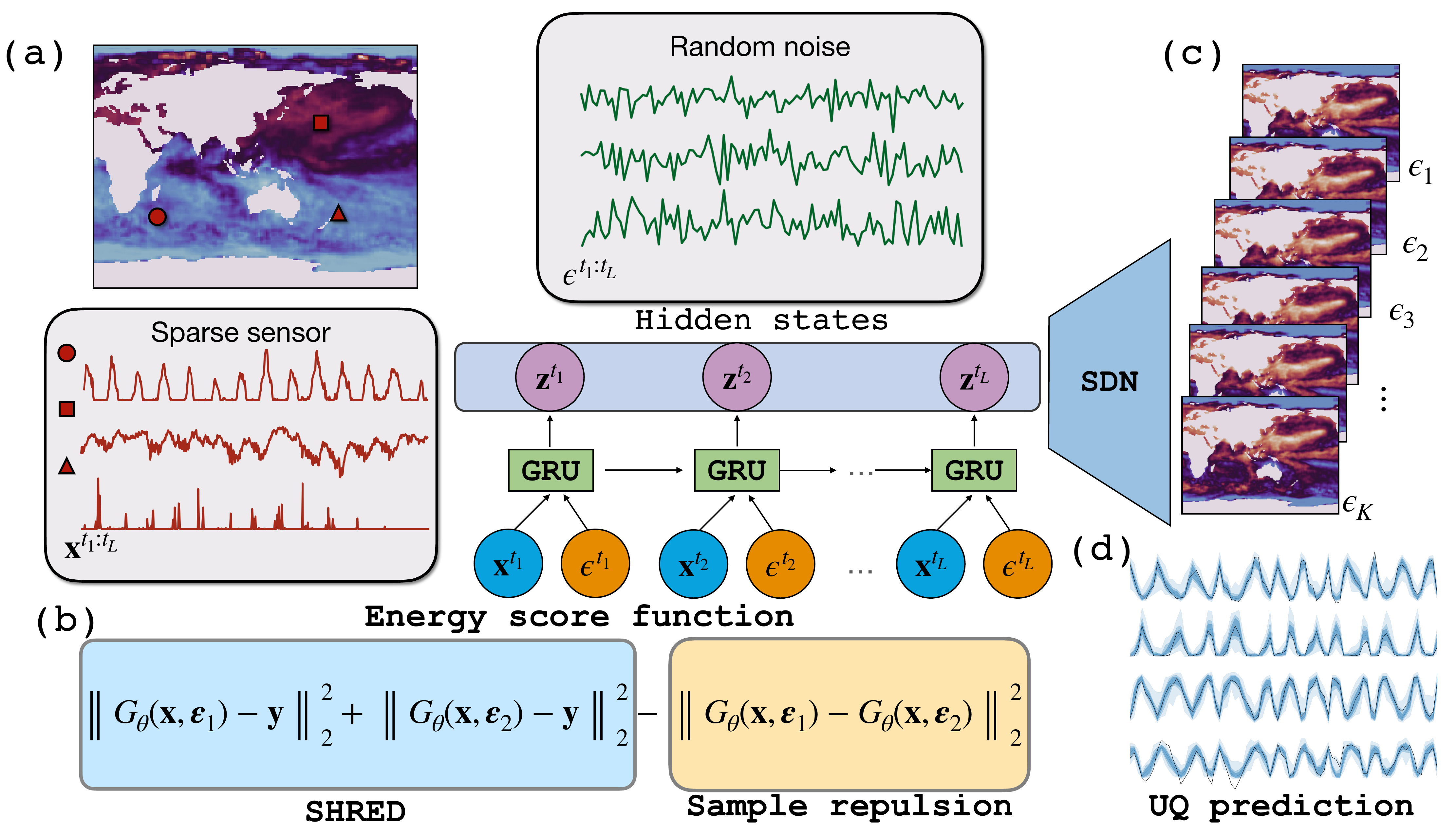}
\caption{(a) UQ-SHRED architecture. The noise vector $\boldsymbol{\varepsilon}$ is concatenated to the sensor input at each lag step and kept constant across the temporal window. 
The random noises are concatenated to the sensor input to the temporal unit (GRU or LSTM), as stated in Eqn.~\ref{eq:uq_shred_decode}. 
(b) The function $G(\cdot)$ represents the SHRED architecture with temporal units and shallow decoder networks. Two independent forward passes with different noise draws produce $G_\theta(\mathbf{x}, \boldsymbol{\varepsilon}_1)$ and $G_\theta(\mathbf{x}, \boldsymbol{\varepsilon}_2)$, which define the energy score loss~\eqref{eq:energy_loss}. (c) At inference, $K$ forward passes with independently resampled noise yield an empirical predictive distribution from which quantile-based confidence intervals are extracted. (d) Uncertainty estimation from the Monte Carlo samples for test sensor locations. Black: ground truth, blue: 95\% uncertainty bounds. }
\label{fig:fig1_uq}
\end{figure}

We use engression to enable uncertainty quantification for SHRED with the following architecture. 
In Figure~\ref{fig:fig1_uq}~(a), we demonstrate the architecture design of UQ-SHRED. 
Recall that SHRED maps a lagged sensor window $\vx_t \in \mathbb{R}^{L \times p}$ to the full state $\vy_t \in \mathbb{R}^m$ through a composition of a temporal unit (GRU or LSTM network) and a shallow decoder.
To apply engression, at each forward pass, a noise vector $\boldsymbol{\varepsilon} \sim \mathcal{N}(\mathbf{0}, \mathbf{I}_{d_\varepsilon})$ is sampled i.i.d. as part of the input. The augmented input $\tilde{\vx}_t = [\vx_t, \boldsymbol{\varepsilon}] $
is fed into the temporal unit. The temporal unit produces a latent representation, and the decoder maps it to the full state:
\begin{align}
G_\theta(\vx, \boldsymbol{\varepsilon}) = f_\mathcal{D}(f_\text{LSTM}([\vx, \boldsymbol{\varepsilon}]); \theta).
\label{eq:uq_shred_decode}
\end{align}

This design is simple but effective which follows the engression algorithm. The training data is injected with random gaussian noise, and the training objective is the following energy score function over the dataset. From~\eqref{eqn:eng_energy_score}, the energy score objective is 
\begin{align}
    \label{eqn:uq_shred_risk}
    \mathcal{R}(\theta) = \E{\lrn{\vy-G_\theta(\vx, \boldsymbol{\varepsilon})}_2 - \frac{1}{2}\lrn{G_\theta(\vx, \boldsymbol{\varepsilon})-G_\theta(\vx, \boldsymbol{\varepsilon'})}_2},
\end{align}
where $\boldsymbol{\varepsilon}, \boldsymbol{\varepsilon}'\sim\mathcal{N}(0, \mathbf{I}_{d_{\varepsilon}})$. 

The energy score function above in Eqn.~\ref{eqn:uq_shred_risk} can be estimated via finite-sample approximation~\cite{shen2023engression} by drawing samples of $\boldsymbol{\varepsilon}\sim\mathcal{N}(0, \mathbf{I}_{d_{\varepsilon}})$ during training. In practice, we approximate this expectation using two i.i.d. noise draws $\boldsymbol{\varepsilon}_1, \boldsymbol{\varepsilon}_2$ and obtain
\begin{align}
    \hat{\vy}_i = f_\mathcal{D}(f_{\text{LSTM}}([\vx,  \boldsymbol{\varepsilon}_i]);\, \theta).
\end{align}
Then, we can compute the training loss 
\begin{equation}
\mathcal{L}_{\mathrm{ES}}(\theta;\, \vx, \vy)
  = \frac{1}{2}\bigl(\|\hat{\vy}_1 - \vy\| + \|\hat{\vy}_2 - \vy\|\bigr)
  - \frac{1}{2}\|\hat{\vy}_1 - \hat{\vy}_2\|.
\label{eq:energy_loss}
\end{equation}
The overall objective is
\begin{equation}
\min_\theta \;\frac{1}{|\mathcal{D}|}\sum_{(\vx, \vy) \in \mathcal{D}} \mathcal{L}_{\mathrm{ES}}(\theta;\, \vx, \vy).
\label{eq:uq_shred_training}
\end{equation}

Training optimizes a single network that not only reconstructs the target from noise-augmented inputs but, through the repulsive term in the energy score, encourages diverse predictions under different noise realizations. 
In this way, the model will address the injected noise by maximizing the distance between different predicted variants, which discourages mode collapse and drives the model to approximate the true conditional distribution of the target. 

\subsection{Inference} 

During inference time, UQ-SHRED can generate samples of $Y\mid X=x$ simply by generating Monte Carlo samples of the noise distribution. 
For a test input sparse sensor trajectory $\vx'$, we have the following steps: 

\begin{enumerate}
    \item Draw $K$ samples from the noise distribution $\boldsymbol{\varepsilon}\overset{\text{i.i.d.}}{\sim}\mathcal{N}(0, \mathbf{I}_{d_{\varepsilon}})$ with $\boldsymbol{\varepsilon}_1, ..., \boldsymbol{\varepsilon}_K$. 
    \item Then $\hat\vy_k=\mathcal{D}(f_\text{LSTM}(\vx', \boldsymbol{\varepsilon}_k)), k=1,...,K$ is a sample of the estimated distribution of $Y\mid X=x$.
    \item We further obtain estimators through empirical samples $\hat{\vy}_1, ..., \hat{\vy}_K$. For example, we can obtain the conditional mean from $\Ep{\boldsymbol{\varepsilon}}{\mathcal{D}(f_\text{LSTM}(\vx', \boldsymbol{\varepsilon}))}$ and conditional $\alpha$-quantile $\mathbb{Q}_\alpha(\mathcal{D}(f_\text{LSTM}(\vx', \boldsymbol{\varepsilon})))$. Both quantities are estimated via plug-in estimators from the Monte Carlo samples of $\boldsymbol{\varepsilon}_k$. 
\end{enumerate}
Figure~\ref{fig:fig1_uq}~(c) illustrates this process of probabilistic inference procedure for uncertainty quantification.

\subsection{Theoretical guarantee} 

In this subsection, we study theoretical guarantee of UQ-SHRED on distributional learning, conditional mean estimation, and quantile learning. 

\begin{theorem}
\label{thm:uq_shred_population}

Suppose we have the following: 
    
(1.a) The sensor-state pair $(\vx,\vy)$ admits a joint distribution, and for almost every $\vx \in \mathbb{R}^{L\times p}$, the conditional law
$\mathcal{L}(\vy \mid \vx)$ is well-defined.

(1.b) The state has finite first moment $\E{\|\vy\|_2} < \infty$.

(1.c) The UQ-SHRED architecture $G_\theta(\vx,\boldsymbol{\varepsilon})$ induces a conditional law $P_\theta(\cdot \mid \vx)$, and the model class is well-specified that there exists $\theta^*$ such that for almost every $\vx$, we have $P_{\theta^*}(\cdot \mid \vx)=\mathcal{L}(\vy\mid \vx)$.  

(1.d) For almost every $\vx\in\mathbb R^{L\times p}$, each induced conditional law $P_\theta(\cdot\mid \vx)$, $\theta\in\Theta$, has finite first moment.

Then, the population minimizer $\hat\theta\in \arg\min_\theta \mathcal{R}(\theta)$ (defined in \eqref{eqn:uq_shred_risk}) satisfies
\begin{align}
    P_{\hat\theta}(\cdot \mid \vx)
=
\mathcal{L}(\vy\mid \vx)
\qquad
\forall\vx\text{ a.e.}
\end{align}
\end{theorem}

Theorem~\ref{thm:uq_shred_population} demonstrates that UQ-SHRED is a valid distributional learning method for this sparse sensing problem. Condition (1.a) assumes that we have sufficient sensor history from the Takens's theorem to reconstruct the spatial states. Condition (1.b) ensures the state is finite. Condition (1.c) assumes that there exists a minimizer within the model class that can represent the true distributional law $\mathcal{L}(\vy\mid\vx)$. Condition (1.d) ensures that the model-induced conditional laws have finite first moments, so that the energy score is well defined and strictly proper.

\begin{proof}

Conditioning on a fixed $\vx$, from the law of total expectation, the population risk can be written as
\begin{align}
    \mathcal{R}(\theta) &= \Ep{\vx}{\mathcal{R}_{\vx}(\theta)} \\
    &= \Ep{\vx}{
\|\vy-G_\theta(\vx,\boldsymbol{\varepsilon})\|_2
-\frac12 \|G_\theta(\vx,\boldsymbol{\varepsilon})-G_\theta(\vx,\boldsymbol{\varepsilon}')\|_2
\,\Big|\, \vx}
\end{align}
Conditional on $X=\vx$, the random variable $G_\theta(\vx,\boldsymbol{\varepsilon})$ follows the
distribution $P_\theta(\cdot\mid \vx)$, while $\vy$ follows the true conditional
law $\mathcal L(\vy\mid X=\vx)$. Therefore $\mathcal R_\vx(\theta)$ coincides with
the energy-score risk of the predictive distribution $P_\theta(\cdot\mid \vx)$
relative to the true conditional distribution $\mathcal L(\vy\mid X=\vx)$.

By condition (1.d), for almost every $\vx$, all $P_\theta(\cdot\mid \vx)$ have finite first moments. Hence the energy score is strictly proper on this pair of conditional laws~\cite{szekely2003statistics,szekely2023energy}.
As the energy score is a strictly proper scoring rule,
the energy-score risk satisfies
\begin{align}
\mathcal R_\vx(Q) \ge \mathcal R_\vx(P_\vx)
\end{align}
for any predictive distribution $Q$, with equality iff $Q=P_\vx$.
Taking $Q = P_\theta(\cdot|\vx)$ therefore implies that
\begin{align}
\mathcal R_\vx(\theta)
\end{align}
is minimized if and only if
\begin{align}
P_\theta(\cdot|\vx)=\mathcal L(\vy|X=\vx).
\end{align}
It means that given $\vx$, we must have $P_\theta(\cdot|\vx)=\mathcal L(\vy|X=\vx)$ to be the minimizer.

We then show further this holds for almost every $\vx$ by contradiction.

By assumption (1.c), there exists $\theta^*$ such that
\begin{align}
P_{\theta^*}(\cdot\mid \vx)=\mathcal L(\vy\mid X=\vx),\;
\text{for almost every } \vx,
\end{align}
so $\theta^*$ minimizes the population risk. 

Now suppose $\hat\theta$ is a population minimizer and
\begin{align}
P_{\hat\theta}(\cdot\mid x)\neq \mathcal L(\vy\mid X=\vx)
\end{align}
on a set of $\vx$ with positive measure, then strict propriety implies that
\begin{align}
\mathcal R_\vx(\hat\theta) > \mathcal R_\vx(\theta^*)
\end{align}
on that set. Taking expectation over $X$ yields
\begin{align}
\mathcal R(\hat\theta) > \mathcal R(\theta^*).
\end{align}

Then this is no longer a population minimizer. Contradiction! 

Therefore, we have
\begin{align}
P_{\hat\theta}(\cdot\mid \vx)=\mathcal L(\vy\mid X=\vx),\;
\text{for almost every } \vx.
\end{align}
\end{proof}

\begin{remark}
Distributional approximation of UQ-SHRED under model misspecification. 
\end{remark}
Because the energy score is a strictly proper scoring rule, the population minimizer under misspecification satisfies
\begin{align}
    \hat\theta \;\in\; \arg\min_{\theta\in\Theta}\; \mathcal{R}(\theta)
    \;=\;
    \arg\min_{\theta\in\Theta}\; \Ep{\vx}{\mathrm{ES}\!\left(P_\theta(\cdot\mid\vx),\; \mathcal{L}(\vy\mid\vx)\right)},
\end{align}
where $\mathrm{ES}(Q, P)$ denotes the expected energy score of the predictive distribution $Q$ evaluated against draws from $P$. Since strict propriety implies that $\mathrm{ES}(Q, P) \ge \mathrm{ES}(P, P)$ with equality if and only if $Q = P$, the population minimizer $P_{\hat\theta}(\cdot\mid\vx)$ is the \emph{closest} element of $\mathcal{P}_\Theta$ to the true conditional law in the sense of minimizing the expected energy score divergence
\begin{align}
    d_{\mathrm{ES}}(Q, P) \;:=\; \mathrm{ES}(Q, P) - \mathrm{ES}(P, P)\;=\;\frac{1}{2}\,\E{\|Z - Z'\|_2} - \E{\|Z - W\|_2} + \frac{1}{2}\,\E{\|W - W'\|_2},
\end{align}
where $Z, Z' \overset{i.i.d.}{\sim} Q$ and $W, W' \overset{i.i.d.}{\sim} P$, which is nonnegative and equals zero if and only if $Q = P$.

Thus, even when condition (1.c) does not hold exactly, UQ-SHRED returns the best distributional approximation within its model class as measured by the energy score. The practical consequence is that the predictive distribution may not achieve exact calibration, but it remains the optimal approximation available to the architecture, and the degree of miscalibration is controlled by the approximation capacity of the model class $\mathcal{P}_\Theta$.

\begin{corollary}
\label{cor:uq_shred_mean}
Suppose that conditions (1.a)-(1.d) of Theorem~\ref{thm:uq_shred_population} hold. 
Additionally, we assume 

(1.e) The state has finite second moment $\E{\|\vy\|_2^2} < \infty$.

Then any population minimizer $\hat\theta$ in Theorem~\ref{thm:uq_shred_population} satisfies
\[
\Ep{P_{\hat\theta}}{\vy \mid \vx}
=
\E{\vy\mid \vx}
\qquad
\text{for almost every } \vx,
\]
and
\[
\mathrm{Cov}_{P_{\hat\theta}}(\vy \mid \vx)
=
\mathrm{Cov}(\vy\mid \vx)
\qquad
\text{for almost every } \vx.
\]
\end{corollary}

\begin{proof}
    This is a direct result of Thm.~\ref{thm:uq_shred_population}. 
    We know under $\theta^*$ that we obtain the true law $\mathcal{L}(\vy\mid\vx)$. Therefore, the expectation and covariance will be the same as the expectation and covariance given the true law, from condition (1.e) with finite second moment.
\end{proof}

\begin{theorem}
\label{thm:monte_carlo}
For a test input $\vx \in \mathbb{R}^{L\times p}$, a coordinate index $j\in\{1,\dots,m\}$, and a level $\alpha\in(0,1)$. We generate $K$ Monte Carlo samples   
$\hat{\vy}^{(k)}
=
f_\mathcal{D}\!\left(f_{\mathrm{LSTM}}([\vx,\boldsymbol{\varepsilon}_k]);\theta\right),\boldsymbol{\varepsilon}_k \overset{i.i.d.}{\sim} \mathcal{N}(0,I_{d_{\varepsilon}}).$
Let $F_{\theta,j}(\cdot\mid \vx)$ denote the conditional CDF of the $j$-th coordinate under the law $P_\theta(\cdot\mid \vx)$. Define the coordinate-wise predictive $\alpha$-quantile by
\[
q_{\theta,j,\alpha}(\vx)
=
\inf\{t\in\mathbb{R}: F_{\theta,j}(t\mid \vx)\ge \alpha\}.
\]

Suppose we have

(2.a) The distribution function $F_{\theta,j}(\cdot\mid \vx)$ is continuous at $q_{\theta,j,\alpha}(\vx)$ and strictly increasing in a neighborhood of $q_{\theta,j,\alpha}(\vx)$.

Define the empirical distribution function and the empirical coordinate-wise $\alpha$-quantile
\begin{align}
\widehat F_{K,j}(t\mid \vx)=\frac{1}{K}\sum_{k=1}^K
\mathbf{1}\{\hat{Y}^{(k)}_j \le t\},\qquad \widehat q_{K,j,\alpha}(\vx)=\inf\{t\in\mathbb{R}: \widehat F_{K,j}(t\mid \vx)\ge \alpha\}.
\end{align}
Then, we have
\begin{align}
\widehat q_{K,j,\alpha}(\vx)
\xrightarrow[]{a.s.}
q_{\theta,j,\alpha}(\vx),
\end{align}
as $K\to\infty$. 
\end{theorem}

Theorem~\ref{thm:monte_carlo} demonstrates that the Monte Carlo inference algorithm UQ-SHRED is a statistically valid quantile learning method. The theorem is based on conditions that $\boldsymbol{\varepsilon}$ samples are i.i.d. and the CDF is continuous.

\begin{proof}
The random variables $\hat{\vy}^{(1)}_j,\dots,\hat{\vy}^{(K)}_j$ 
are i.i.d. draws given $\vx$ from the one-dimensional predictive distribution with CDF $F_{\theta,j}(\cdot\mid \vx)$ as $\boldsymbol{\varepsilon}_i$'s are i.i.d. By Glivenko-Cantelli theorem~\cite{van2000asymptotic}, we have
\begin{align}
\sup_{t\in\mathbb{R}}
\left|
\widehat F_{K,j}(t\mid \vx)-F_{\theta,j}(t\mid \vx)
\right|
\xrightarrow[]{a.s.} 0.
\end{align}
For all $\delta>0$, from (2.a), we have both
\begin{align}
    F_{\theta,j}(q_{\theta,j,\alpha}(\vx)-\delta \mid \vx) < \alpha,\;
F_{\theta,j}(q_{\theta,j,\alpha}(\vx)+\delta \mid \vx) > \alpha
\end{align}
for sufficiently small $\delta>0$. Since $\widehat F_{K,j} \overset{a.s.}{\to} F_{\theta,j}$, we know there exists sufficiently large $K$ that
\begin{align}
\widehat F_{K,j}(q_{\theta,j,\alpha}(\vx)-\delta \mid \vx) < \alpha,\;
\widehat F_{K,j}(q_{\theta,j,\alpha}(\vx)+\delta \mid \vx) > \alpha,\;a.s.
\end{align}
By the definition of the empirical quantile, we know
\begin{align}
q_{\theta,j,\alpha}(\vx)-\delta
\le
\widehat q_{K,j,\alpha}(\vx)
\le
q_{\theta,j,\alpha}(\vx)+\delta,\;a.s.
\end{align}
Since $\delta>0$ can be arbitrary small, under $K\to\infty$, we have
\begin{align}
\widehat q_{K,j,\alpha}(\vx)
\xrightarrow[]{a.s.}
q_{\theta,j,\alpha}(\vx).
\end{align}
\end{proof}

\section{Experiments}
\label{sec:experiments}

\subsection{Evaluation metrics}
\label{sec:metrics}

We assess the quality of the predictive distribution using three complementary metrics.\\

\noindent
\textbf{Calibration.} For each nominal confidence level $\alpha \in (0,1)$, we compute the fraction of ground-truth values that fall within the corresponding $\alpha$-level prediction interval, averaged over all spatiotemporal locations. We denote this observed coverage $\hat{\alpha}$. A well-calibrated model satisfies $\hat{\alpha} \approx \alpha$ for all levels, which we visualise as a calibration diagram plotting $\hat{\alpha}$ versus $\alpha$; the diagonal represents ideal calibration.\\

\noindent
\textbf{Continuous ranked probability score.} The CRPS~\cite{gneiting2007strictly} is a proper scoring rule that jointly assesses calibration and sharpness. For a set of predictive samples $\{\hat{y}^{(k)}\}_{k=1}^K$ and an observation $y$, the CRPS admits the representation
\begin{equation}
\mathrm{CRPS} = \frac{1}{K}\sum_{k=1}^K |\hat{y}^{(k)} - y| \;-\; \frac{1}{2K^2}\sum_{k=1}^K\sum_{j=1}^K |\hat{y}^{(k)} - \hat{y}^{(j)}|,
\label{eq:crps}
\end{equation}
which we average over all spatial locations and time steps. Lower values indicate better probabilistic forecasts.\\

\noindent
\textbf{Sharpness.} The average width of the prediction interval at each confidence level, measuring the concentration of the predictive distribution. Conditional on correct calibration, narrower intervals (sharper forecasts) are preferred.\\

\noindent
\textbf{Interval score.} For a central prediction interval $[\ell, u]$ at nominal level $1-\alpha$ and an observation $y$, the interval score~\cite{gneiting2007strictly} is
\begin{equation}
\mathrm{IS}_{1-\alpha} = (u-\ell) + \frac{2}{\alpha}(\ell-y)\,\mathbf{1}\{y<\ell\} + \frac{2}{\alpha}(y-u)\,\mathbf{1}\{y>u\},
\label{eq:interval_score}
\end{equation}
averaged over all spatiotemporal locations. It penalises interval width and miscoverage jointly, so that intervals which are narrow but fail to contain the observation are not rewarded. Lower values are preferred.\\

\noindent
\textbf{Calibration error.} A scalar summary of the calibration diagram, defined as the mean absolute deviation of observed from nominal coverage, $|A|^{-1}\sum_{\alpha \in A} |\hat{\alpha} - \alpha|$, over the set $A$ of nominal levels reported and expressed in percentage points. The observed coverages $\hat{\alpha}$ are first averaged over random initialisations, so that the calibration error is recoverable from the per-level coverages reported in each table.\\

\begin{figure}[t]
\centering    
\includegraphics[width=\columnwidth]{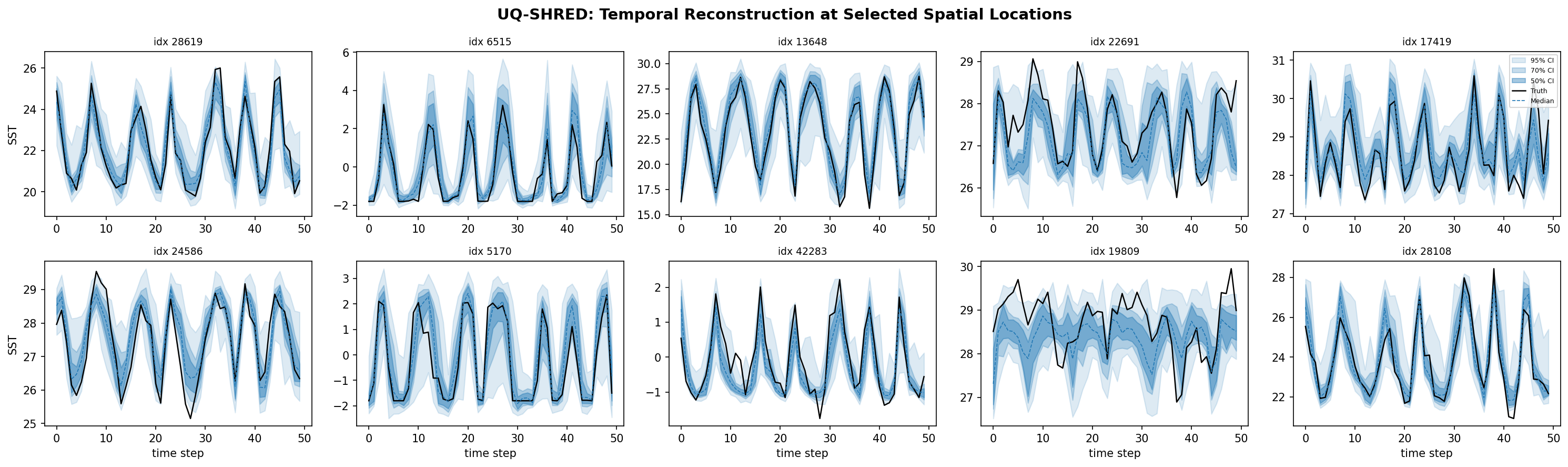}
\caption{UQ-SHRED reconstruction of the SST dataset. Shaded regions: $50\%$, $70\%$, and $95\%$ confidence intervals; blue dashed: median; black solid: ground truth.}
\label{fig:uq_sst_timeseries}
\end{figure}

\subsection{Sea-surface temperature data}

\paragraph{Dataset} The Sea-surface temperature (SST) dataset contains $1,400$ snapshots of weekly averaged sea surface temperature from 1992 to 2019 reported by NOAA~\cite{reynolds2002improved}. The dataset is a reanalysis dataset through interpolation combining in-situ buoy, ship, and satellite
429 observations and every snapshot includes a $180\times 360$ grid with $44,219$ grid points corresponding to the sea-surface. 

We train UQ-SHRED with global SST data from only 3 random sensors, with temporal lag $L=52$ (approximately one year of history), noise dimension $d_{\varepsilon}=1000$, and $200$ Monte Carlo samples at inference to estimate the quantiles. Figure~\ref{fig:uq_sst_timeseries} presents reconstruction from randomly selected sensors and uncertainty quantification on test data. The median prediction tracks the ground truth (black) across both fields. The confidence bands widen at the global temperature shifts, where the rate of change is highest and sparse sensors provide less constraint on the global field.

\begin{figure}[t]
\centering    
\includegraphics[width=\columnwidth]{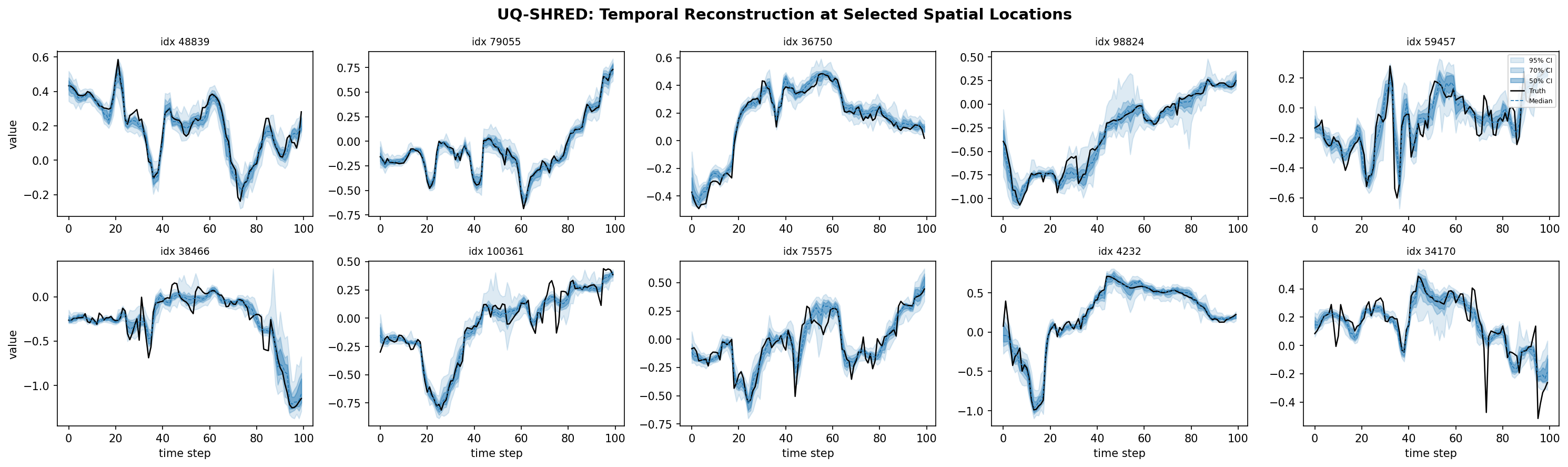}
\caption{UQ-SHRED reconstruction of the isotropic turbulent flow. Shaded regions: $50\%$, $70\%$, and $95\%$ confidence intervals; blue dashed: median; black solid: ground truth.}
\label{fig:uq_iso_timeseries}
\end{figure}

\subsection{Isotropic turbulent flow}

\paragraph{Dataset} The isotropic turbulent flow dataset is a direct numerical simulation of forced isotropic turbulence from the Johns Hopkins Turbulence Database (JHTDB)~\cite{li2008public}. It is generated from the incompressible Navier-Stokes equation using pseudo-spectral method with periodic boundary conditions on a $1024^3$ grid. We extract 2D $350\times 350$ spatial slices of the pressure field, yielding $1,667$ temporal samples.

We train UQ-SHRED with 3 random sensors, $L=100$, noise dimension $d_{\varepsilon}=100$, and $50$ Monte Carlo samples at inference. Figure~\ref{fig:uq_iso_timeseries} presents reconstruction and uncertainty quantification on test data. The confidence bands widen during rapid local pressure fluctuations, reflecting increased reconstruction ambiguity when the flow undergoes violent spatiotemporal changes.

\begin{figure}[t]
\centering    
\includegraphics[width=\columnwidth]{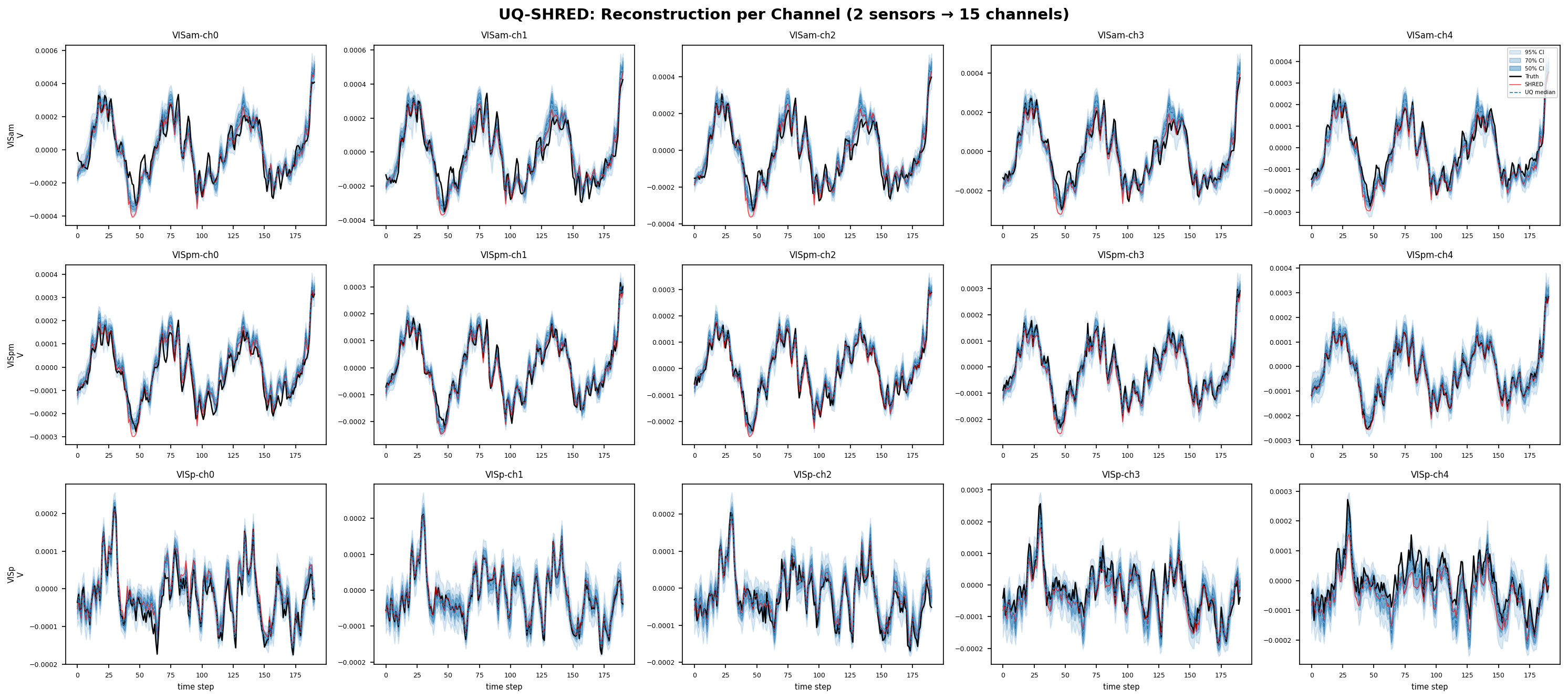}
\caption{UQ-SHRED reconstruction of the neural activity dataset. Shaded regions: $50\%$, $70\%$, and $95\%$ confidence intervals; blue dashed: median; black solid: ground truth.}
\label{fig:uq_neuro_timeseries}
\end{figure}

\subsection{Neural activity data}

\paragraph{Dataset} 

The neural data was obtained from the Allen Institute Visual Coding---Neuropixels project \cite{siegle2021survey}. This dataset contains extracellular electrophysiology recordings from the mouse visual cortex and thalamus during a variety of visual stimuli. We analyzed 0.5 kHz local field potential (LFP) recordings from the anteromedial (VISam), posteromedial (VISpm), and primary (VISp) visual cortex during presentation of a 2-second $180^{\circ}$ drifting grating stimulus. Rather than recording activity from a single unit, LFPs measure population activity from ensembles of neurons \cite{kajikawa2011local}. While LFPs are valuable for studying network dynamics, they require invasive recording techniques and are susceptible to noise and interference, which introduce additional ambiguity into the signals \cite{einevoll2013modelling}. Our aim is to train UQ-SHRED on this dataset to reconstruct neural population activity from a limited number of sensors while providing valid uncertainty quantification.

We train UQ-SHRED on LFP data with 2 randomly selected channels, $L=100$, noise dimension $d_{\varepsilon}=50$, and $100$ Monte Carlo samples at inference. Figure~\ref{fig:uq_neuro_timeseries} presents reconstruction and uncertainty quantification on test data. The confidence bands widen at high-frequency local perturbations in the neural signal, consistent with the increased ambiguity introduced by noise and interference at finer temporal scales.

\subsection{Solar activity data}

\paragraph{Dataset} The solar activity data describes active physical and chemical reactions on the sun. We obtain the solar data from publicly available dataset from NASA. The Solar Dynamics Observatory collects observations of the sun, and we particularly focus on the wavelength 171 Ångstroms, which is known to be the most active and commonly used part of the sun. 
It specifically highlights a spectral line emitted by iron atoms that have
lost 8 electrons at temperatures of 600,000 K. We collected 274 frames of an active area of the sun with $30$ minutes temporal gap of each frame.

We train UQ-SHRED with 3 randomly selected sensors, $L=75$, noise dimension $d_{\varepsilon}=200$, and $100$ Monte Carlo samples at inference. Figure~\ref{fig:uq_sun_timeseries} presents reconstruction and uncertainty quantification on test data. The confidence bands widen during rapid changes in solar irradiance, particularly around transient flare events.

\begin{figure}[t]
\centering    
\includegraphics[width=\columnwidth]{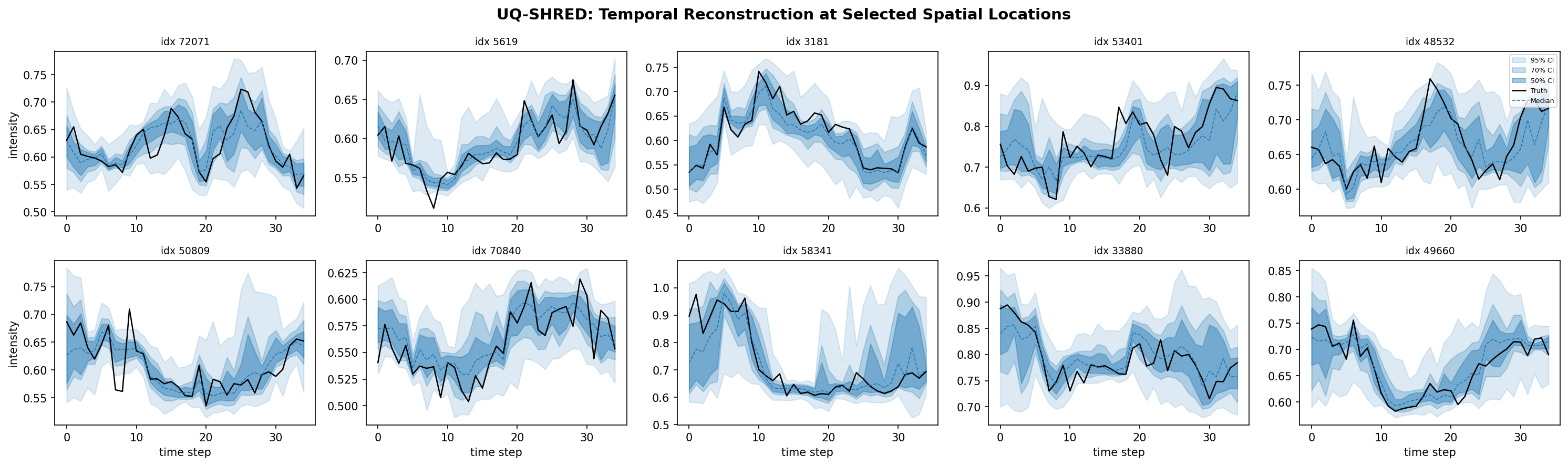}
\caption{UQ-SHRED reconstruction of the solar activity dataset. Shaded regions: $50\%$, $70\%$, and $95\%$ confidence intervals; blue dashed: median; black solid: ground truth.}
\label{fig:uq_sun_timeseries}
\end{figure}

\begin{figure}[t]
\centering
\begin{subfigure}[b]{0.23\columnwidth}
\includegraphics[width=\linewidth]{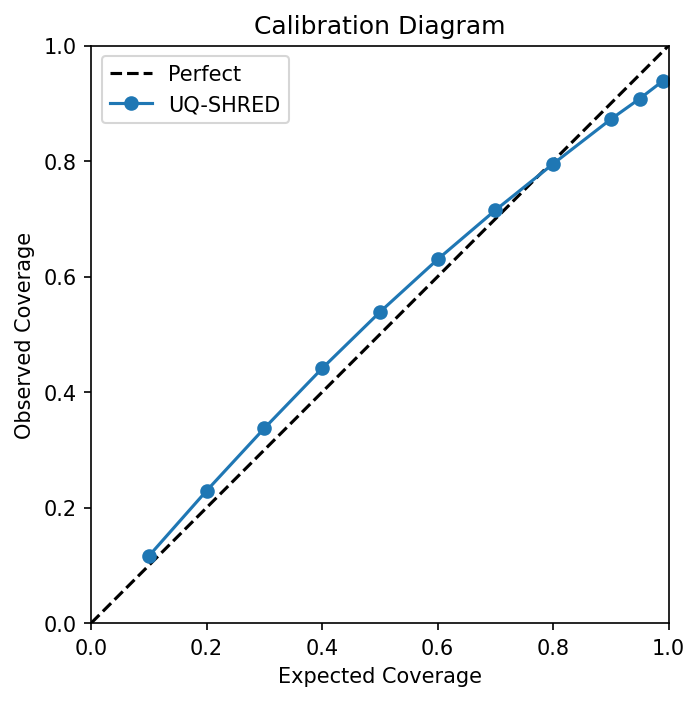}
\caption{SST}
\label{fig:uq_sst_calibration}
\end{subfigure}
\begin{subfigure}[b]{0.23\columnwidth}
\includegraphics[width=\linewidth]{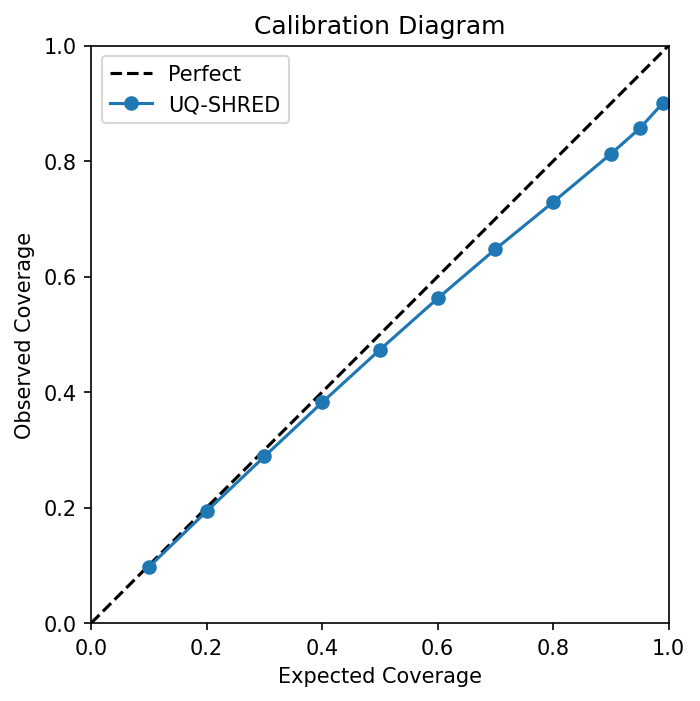}
\caption{Turbulent flow}
\label{fig:uq_iso_calibration}
\end{subfigure}
\begin{subfigure}[b]{0.23\columnwidth}
\includegraphics[width=\linewidth]{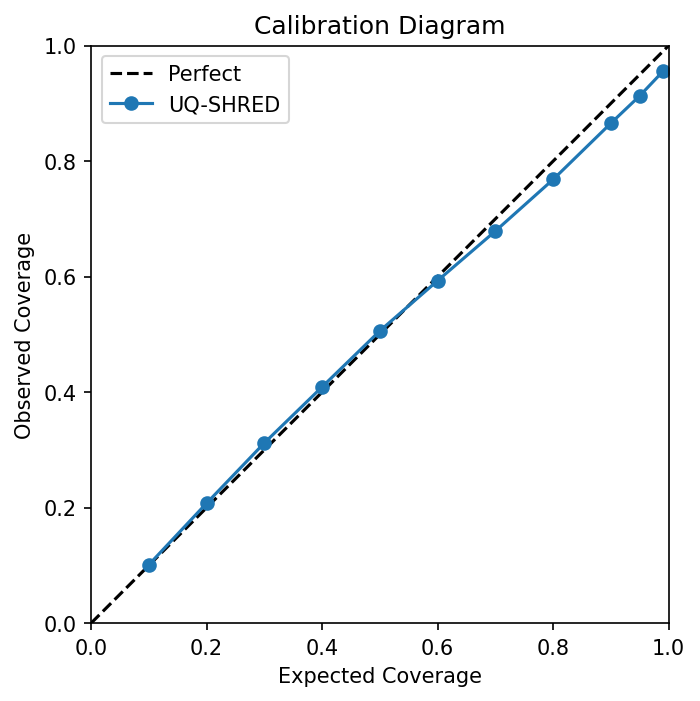}
\caption{Neural activity}
\label{fig:uq_neuro_calibration}
\end{subfigure}
\begin{subfigure}[b]{0.23\columnwidth}
\includegraphics[width=\linewidth]{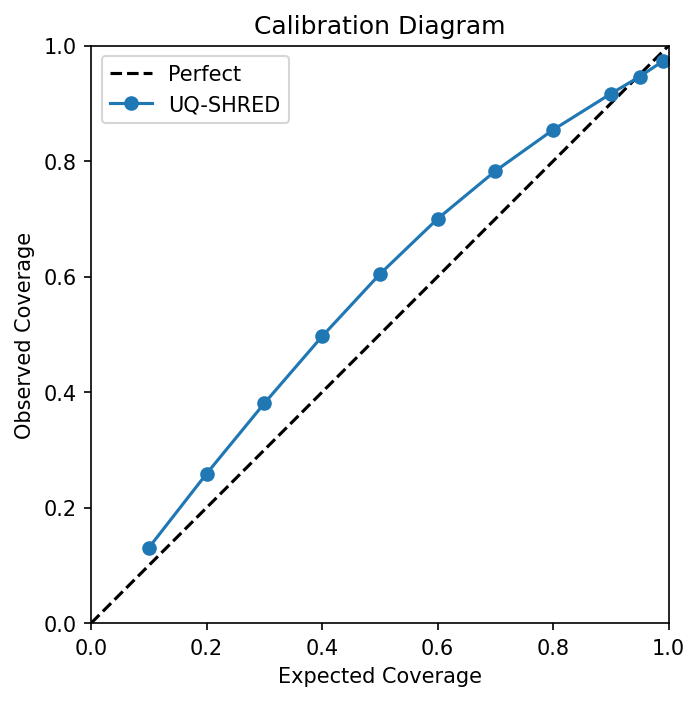}
\caption{Solar activity}
\label{fig:uq_sun_calibration}
\end{subfigure}
\caption{Calibration diagrams for UQ-SHRED across all four datasets. Dashed line: perfect calibration. Blue: UQ-SHRED empirical coverage.}
\label{fig:uq_calibration_all}
\end{figure}

\begin{table}[t]
\centering
\caption{UQ-SHRED metrics across the four experimental datasets. Calibration reports observed coverage (\%) at each nominal confidence level. CRPS represents the continuous ranked probability score. Sharpness is the average CI width.}
\label{tab:uq_metrics}
\footnotesize
\setlength{\tabcolsep}{3.5pt}
\renewcommand{\arraystretch}{1.15}
\begin{tabular}{lccccccccc}
\toprule
 \textbf{Dataset} & \textbf{RMSE} & \multicolumn{5}{c}{\textbf{Calibration \%} (expected to observed)} & & \multicolumn{2}{c}{\textbf{Sharpness}} \\
\cmidrule(lr){2-2} \cmidrule(lr){3-7} \cmidrule(lr){9-10}
 &  & 50\% & 70\% & 90\% & 95\% & 99\% & \textbf{CRPS} & 95\% CI & 50\% CI \\
\midrule
SST & 0.656 & 53.9 & 71.5 & 87.3& 90.8 & 93.9 & 0.343 & 2.791 & 0.890 \\
\midrule
ISO & 0.656 & 47.3 & 64.7 & 81.2& 85.7 & 90.1 & 0.032 & 0.209 & 0.067 \\
\midrule
Neural & $3.54e^{-5}$ & 50.6 & 67.9 & 86.6& 91.3 & 95.5 & $1.8e^{-5}$ & $1.14e^{-4}$ & $3.9e^{-5}$ \\
\midrule
Solar & $0.029$ & 60.4 & 78.3 & 91.7 & 94.6 & 97.3 & $0.016$ & $0.139$ & $0.051$ \\
\bottomrule
\end{tabular}
\end{table}

\paragraph{Summary across datasets.}

Table~\ref{tab:uq_metrics} reports the calibration, CRPS, and sharpness metrics for all four datasets. Across all experiments, the median prediction of UQ-SHRED tracks the ground truth closely. The observed coverages at the $50\%$, $70\%$, $90\%$, $95\%$, and $99\%$ nominal levels are consistent with well-calibrated predictive distributions, as shown in the calibration diagrams in Figure~\ref{fig:uq_calibration_all}. In each case, the confidence intervals are not uniformly wide but vary spatially and temporally: they widen where reconstruction ambiguity is highest and narrow where sensor data sufficiently constrains the state. Additional spatial reconstruction visualizations are provided in Appendix~\ref{app:spatial_figures}.

\begin{figure}[t]
\centering
\begin{subfigure}[b]{\columnwidth}
    \centering
    \includegraphics[width=\columnwidth]{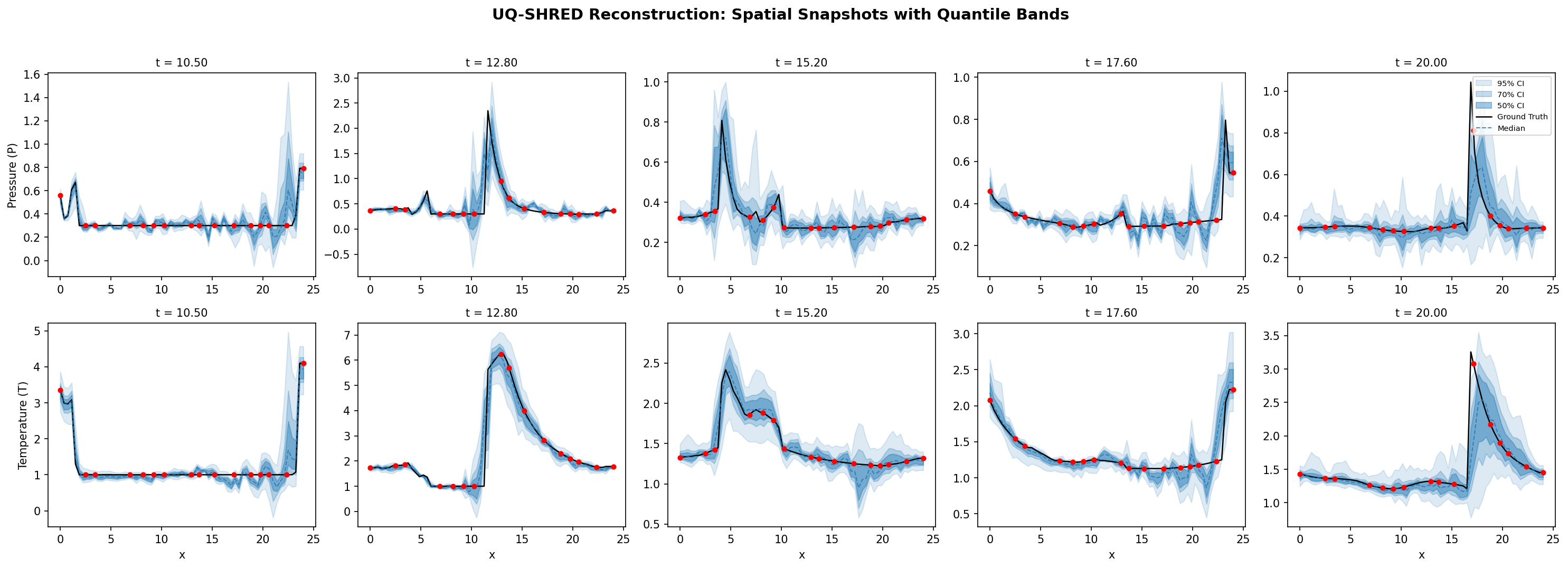}
    \caption{Closer trajectory (Run~0, small IC perturbation from training data).}
    \label{fig:uq_1drde_closer_snapshots}
\end{subfigure}
\vspace{0.3em}
\begin{subfigure}[b]{\columnwidth}
    \centering
    \includegraphics[width=\columnwidth]{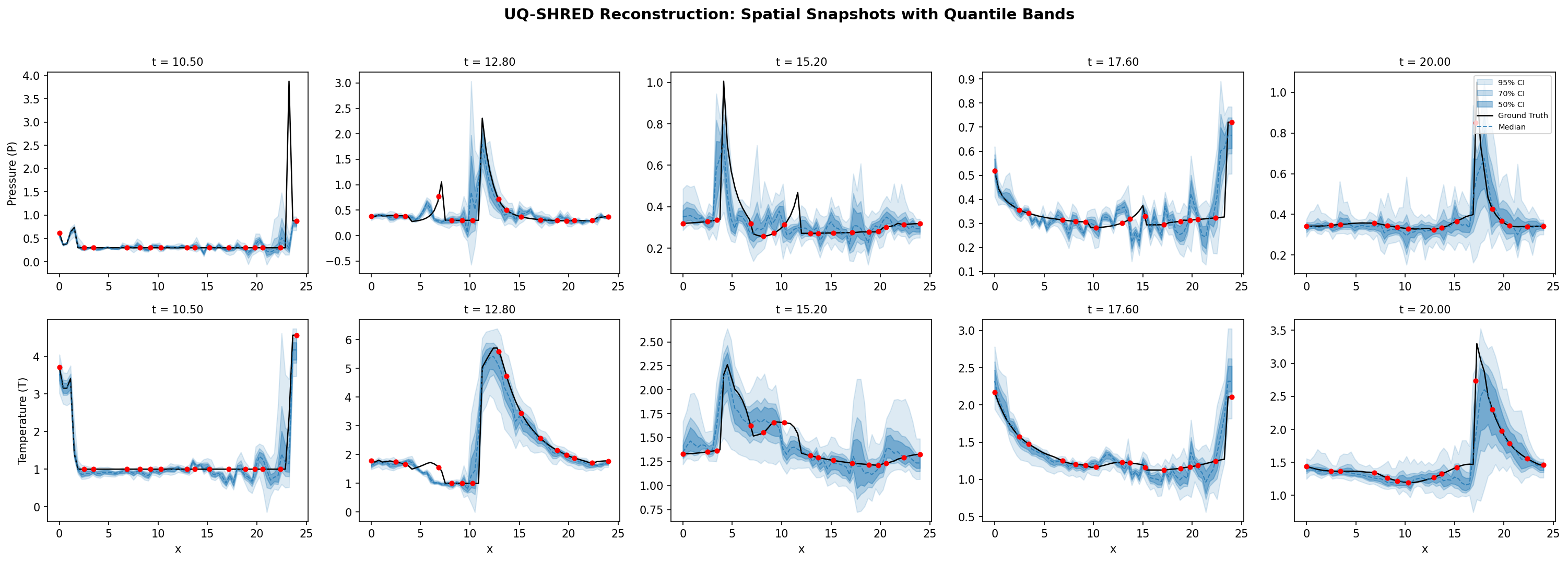}
    \caption{Further trajectory (Run~1, larger IC perturbation).}
    \label{fig:uq_1drde_further_snapshots}
\end{subfigure}
\caption{UQ-SHRED reconstruction of the 1D RDE transient stage at selected time instances. Top rows: pressure; bottom rows: temperature. Shaded regions: $50\%$, $70\%$, and $95\%$ confidence intervals; red dashed: median; black solid: ground truth; red dots: sensor locations.}
\label{fig:uq_1drde_snapshots}
\end{figure}

\subsection{1D RDE transient Stage}
\label{sec:uq_1drde}

\paragraph{Dataset}
We consider the ignition transient of a one-dimensional rotating detonation engine (RDE) generated by Koch's reduced-order model~\cite{koch2020modeling}, which approximates detonation-wave propagation in an annular combustor via the reactive Euler equations with source terms for injection, mixing, and finite-rate chemistry. The transient stage---in which the flow transitions from a quiescent premixed state to sustained detonation---produces sharp, rapidly evolving pressure and temperature fronts whose positions and amplitudes vary across initial conditions. We reconstruct both the pressure $P(x,t)$ and temperature $T(x,t)$ fields over the temporal window $t \in [10.0, 20.0]$ ($T = 101$ frames) on a downsampled spatial grid of $m_x = 64$ points.

We train UQ-SHRED from scratch on one ground-truth trajectory using the energy score loss, with $p = 16$ uniformly spaced sensors, noise dimension $d_\varepsilon = 50$, and $200$ Monte Carlo samples at inference. 
The simulation ensemble comprises five held-out trajectories whose initial-condition perturbations increase in magnitude from the training trajectory. Here we focus on two representative cases: a \emph{closer} trajectory (Run~0, smallest perturbation) and a \emph{further} trajectory (Run~1, larger perturbation), which together test how the learned uncertainty responds to distributional shift; the full set of five is used in the baseline comparison of Sec.~\ref{sec:baselines}.

Figure~\ref{fig:uq_1drde_snapshots} presents spatial snapshots at selected time instances for both cases. The median prediction tracks the ground truth across both fields, while the $50\%$, $70\%$, and $95\%$ confidence bands widen at the detonation wave fronts---where sparse sensors provide the least constraint on the full-field reconstruction---and narrow in the quiescent post-detonation regions where the state is well-determined. Table~\ref{tab:uq_1drde} reports the calibration, CRPS, and sharpness metrics for both test cases alongside the deterministic SHRED baseline RMSE. Additional figures are provided in Appendix~\ref{app:uq_1drde}.

\begin{figure}[t]
\centering
\begin{subfigure}[b]{0.48\columnwidth}
    \centering
    \includegraphics[width=\linewidth]{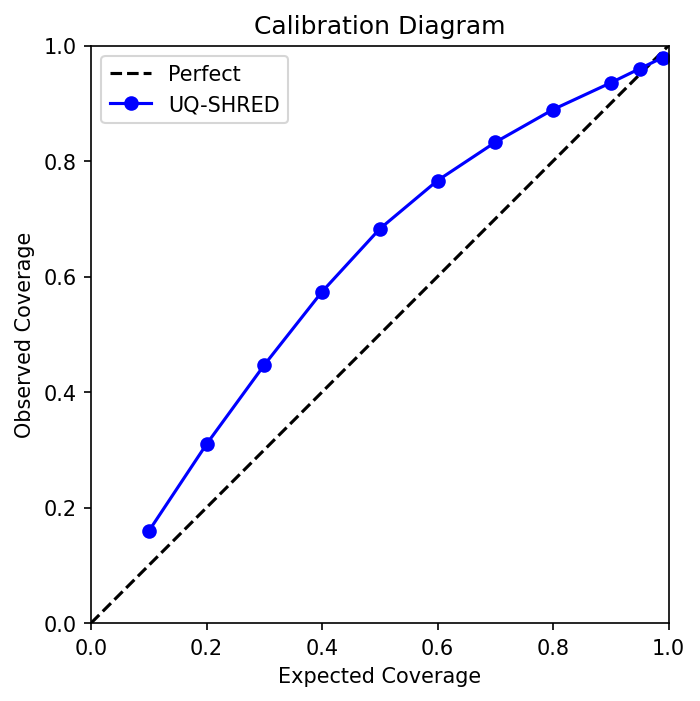}
    \caption{Closer trajectory (Run~0).}
\end{subfigure}
\hfill
\begin{subfigure}[b]{0.48\columnwidth}
    \centering
    \includegraphics[width=\linewidth]{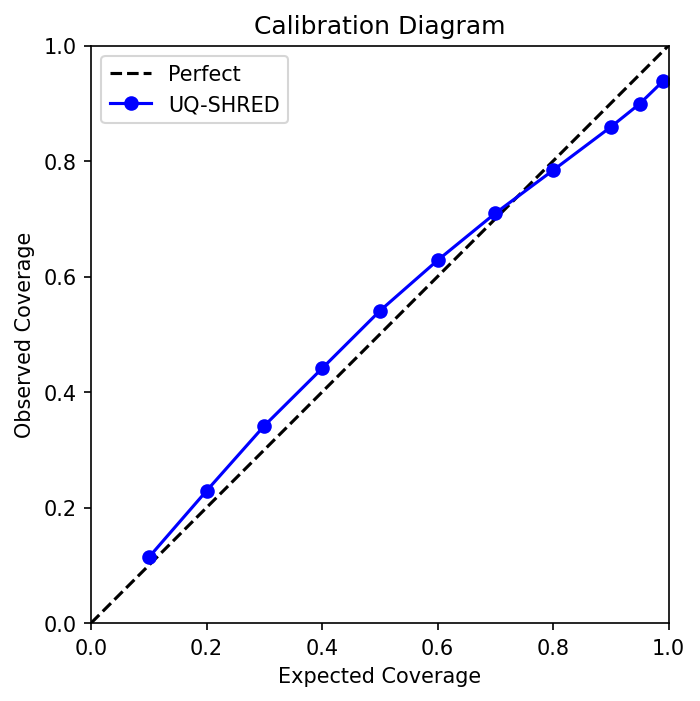}
    \caption{Further trajectory (Run~1).}
\end{subfigure}
\caption{Calibration diagrams for UQ-SHRED on the 1D RDE transient stage. Dashed line: perfect calibration. Blue: UQ-SHRED empirical coverage.}
\label{fig:uq_1drde_calibration}
\end{figure}

\begin{table}[t]
\centering
\caption{UQ-SHRED metrics on the 1D RDE transient stage for two test trajectories with different initial-condition perturbation magnitudes. Calibration reports observed coverage (\%) at each nominal confidence level. CRPS represents the continuous ranked probability score. Sharpness is the average CI width.}
\label{tab:uq_1drde}
\footnotesize
\setlength{\tabcolsep}{3.5pt}
\renewcommand{\arraystretch}{1.15}
\resizebox{\textwidth}{!}{% scaled to fit the Proc. R. Soc. A text width
\begin{tabular}{lcccccccccc}
\toprule
 & \multicolumn{2}{c}{\textbf{RMSE (median)}} & \multicolumn{5}{c}{\textbf{Calibration \%} (expected to observed)} & & \multicolumn{2}{c}{\textbf{Sharpness}} \\
\cmidrule(lr){2-3} \cmidrule(lr){4-8} \cmidrule(lr){10-11}
\textbf{Test case} & $P$ & $T$ & 50\% & 70\% & 90\% & 95\% & 99\% & \textbf{CRPS} & 95\% CI & 50\% CI \\
\midrule
Closer (Run~0) & 0.156 & 0.307 & 69.5 & 83.5 & 93.6 & 95.9 & 98.0 & 0.070 & 0.556 & 0.185 \\
Further (Run~1) & 0.194 & 0.398 & 56.9 & 72.9 & 88.0 & 91.5 & 95.1 & 0.096 & 0.569 & 0.190 \\
\midrule
\multicolumn{11}{l}{\textit{Deterministic SHRED baseline:}} \\
Closer (Run~0) & 0.145 & 0.274 & \multicolumn{5}{c}{---} & --- & --- & --- \\
Further (Run~1) & 0.198 & 0.389 & \multicolumn{5}{c}{---} & --- & --- & --- \\
\bottomrule
\end{tabular}}
\end{table}

\begin{figure}[t]
\centering
\includegraphics[width=\columnwidth]{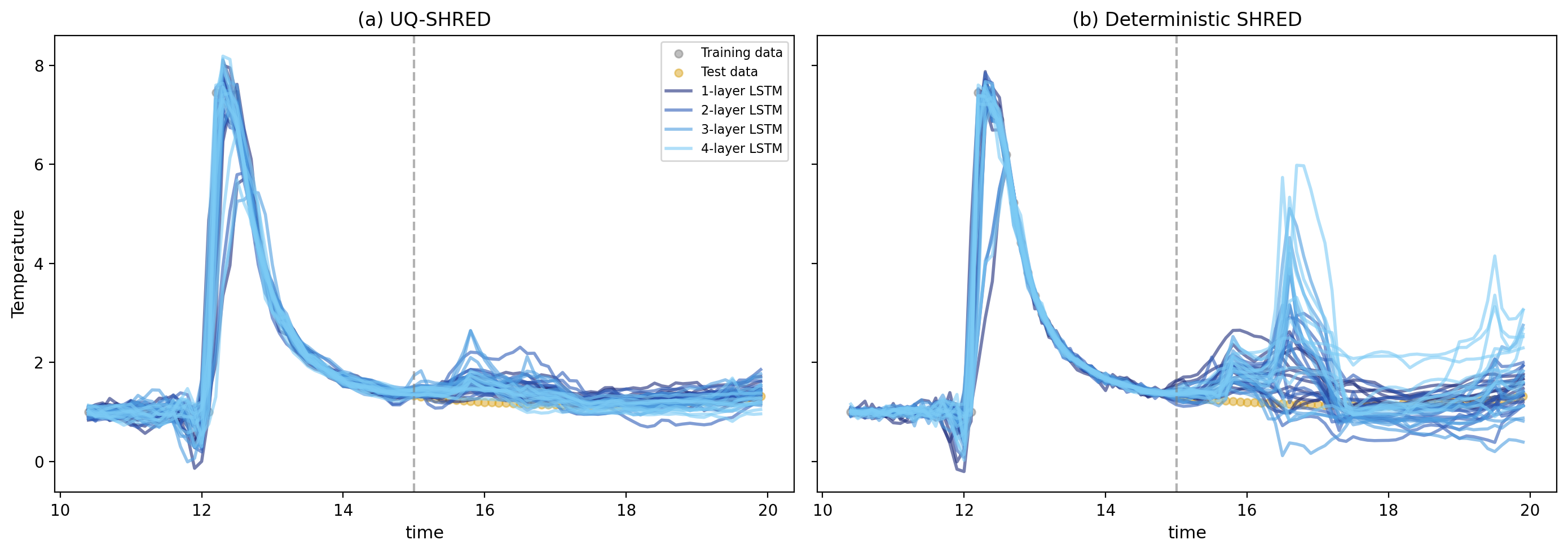}
\caption{Stability comparison between UQ-SHRED and deterministic SHRED under temporal extrapolation on the 1D RDE data. Visualizations are presented for temporal traces at a selected spatial location $x=14.50$ on the temperature field.}
\label{fig:uq_vs_det_stability}
\end{figure}

To further investigate the stability of UQ-SHRED under temporal extrapolation, we conduct an experiment in which both UQ-SHRED and deterministic SHRED are trained on the first portion of the ground-truth trajectory ($t \leq 15.0$) and evaluated on the remaining unseen portion ($t > 15.0$). Both methods utilize the same SHRED-architecture with number of LSTM layers varying from 1 to 4, with 10 independent random initializations for each architecture. Figure~\ref{fig:uq_vs_det_stability} presents the visualizations: both methods produce comparable reconstructions within the training window, but beyond the training boundary (dashed vertical line), the deterministic SHRED predictions diverge substantially across initializations---the variance is on average $3.30\times$ larger than that of UQ-SHRED. This behavior parallels the findings of Shen and Meinshausen~\cite{shen2023engression}, who demonstrated that engression produces stable out-of-support predictions where standard regression with identical architectures yields highly variable extrapolations. In the present setting, the energy score loss regularizes the learned mapping so that different initializations converge to consistent predictive distributions, whereas the MSE objective---which is unconstrained beyond the training support---permits each initialization to extrapolate differently. The result suggests that engression not only provides calibrated uncertainty estimates but also confers improved robustness when extrapolating beyond the training time window.

Figure~\ref{fig:uq_1drde_calibration} shows the calibration diagrams with physically interpretable behavior: the model's learned uncertainty bandwidth, calibrated on data near the training distribution, becomes slightly insufficient when wavefront positions shift substantially.

\subsection{Comparison with uncertainty quantification baselines}
\label{sec:baselines}

To place the calibration reported above in context, we compare UQ-SHRED on the 1D RDE transient stage of Sec.~\ref{sec:uq_1drde} against one representative of each family surveyed in Sec.~\ref{sec:prelim}---MC-dropout~\cite{gal2016dropout}, deep ensembles~\cite{lakshminarayanan2017simple}, a heteroscedastic Gaussian model~\cite{kendall2017uncertainties}, and a last-layer Laplace approximation~\cite{daxberger2021laplace} as a tractable surrogate for full weight-space Bayesian inference---together with quantile regression~\cite{koenker1978regression} and split conformal prediction~\cite{lei2021conformal} (CQR). All methods consume identical data, splits, scalers and sensor sets, use the same SHRED architecture and training schedule, and are evaluated with the same number of Monte Carlo draws, so that the reported differences are attributable to the uncertainty model rather than to network capacity or training effort. Reconstruction accuracy lies within $8\%$ across all methods and is reported in Appendix~\ref{app:baselines} together with the protocol and per-level tables; Table~\ref{tab:baselines_main} compares calibration, interval score and computational cost on the closer and further test trajectories of Sec.~\ref{sec:uq_1drde}.

\begin{table}[t]
\centering
\caption{Comparison on the 1D RDE transient stage, evaluated on the closer (Run~0) and further (Run~1) test trajectories, which differ in the magnitude of their initial-condition perturbation from the training trajectory. Cal.\ err. and IS$_{95}$ are the calibration error and the interval score at the $95\%$ level, both defined in Sec.~\ref{sec:metrics}; Nets is the number of networks required at inference; $\Delta$par.\ is the additional parameters relative to a deterministic SHRED model at the state dimension of the sea-surface temperature field ($m = 44{,}219$). Values are means over ten random initialisations; Cal.\ err. is computed from the coverages averaged over initialisations. Lower is better throughout.}
\label{tab:baselines_main}
\footnotesize
\setlength{\tabcolsep}{4.5pt}
\renewcommand{\arraystretch}{1.2}
\begin{tabular}{lccccrr}
\toprule
 & \multicolumn{2}{c}{\textbf{Closer (Run~0)}} & \multicolumn{2}{c}{\textbf{Further (Run~1)}} & & \\
\cmidrule(lr){2-3} \cmidrule(lr){4-5}
\textbf{Method} & Cal.\ err. & IS$_{95}$ & Cal.\ err. & IS$_{95}$ & \textbf{Nets} & $\Delta$\textbf{par.} \\
\midrule
SHRED (deterministic)      & --- & --- & --- & --- & 1  & --- \\
MC-dropout                 &  5.5 & 1.13 & 12.3 & 2.10 & 1  & 0 \\
Deep ensemble ($M{=}10$) & 13.4 & \textbf{0.74} & 7.0 & 1.41 & 10 & $3.39\times10^{8}$ \\ 
Heteroscedastic Gaussian & 3.8 & 0.92 & 18.1 & 1.89 & 1 & $1.77\times10^{7}$ \\
Last-layer Laplace         & 14.8 & 1.46 & 10.7 & 2.24 & 1  & 0 \\
Quantile regression + CQR & 3.6 & 1.68 & 14.9 & 2.54 & 1 & $1.77\times10^{8}$ \\
SHRED + split conformal    &  3.6 & 2.17 &  7.5 & 3.04 & 1  & 0 \\
\midrule
\textbf{UQ-SHRED} & 7.7 & 0.86 & \textbf{4.0} & \textbf{1.34} & 1 & $\mathbf{1.28\times10^{4}}$ \\ 
\textbf{UQ-SHRED + conformal} & \textbf{2.0} & 1.22 & 6.9 & 1.67 & 1 & $1.28\times10^{4}$ \\
\bottomrule
\end{tabular}
\end{table}

\paragraph{Closer trajectory}
On Run~0, UQ-SHRED attains an interval score of $0.86$ and a calibration error of $7.7$ percentage points from a single network. Several methods achieve lower calibration error on this near-training trajectory: split conformal attains $3.6$ percentage points, benefiting from a coverage guarantee that holds exactly when the calibration and test data are exchangeable, a condition approximately satisfied here, and the heteroscedastic Gaussian and quantile regression attain $3.8$ and $3.6$ respectively. The deep ensemble attains a lower interval score of $0.74$, though at the cost of ten independently trained networks. Conformalising the UQ-SHRED quantiles (Appendix~\ref{app:baselines}) reduces its calibration error to $2.0$ percentage points, still from a single network, at the cost of wider intervals. As shown below, these advantages narrow or reverse under distributional shift.

\paragraph{Further trajectory}
On Run~1, whose initial condition is perturbed more strongly from the training trajectory, UQ-SHRED attains the smallest calibration error and the lowest interval score of all methods considered---$4.0$ percentage points and $1.34$, against $7.0$ and $1.41$ for the deep ensemble, $7.5$ and $3.04$ for split conformal, and $18.1$ and $1.89$ for the heteroscedastic Gaussian. 
To assess whether this advantage is specific to Run~1, Figure~\ref{fig:baselines_rde} evaluates every method on all five held-out trajectories introduced in Sec.~\ref{sec:uq_1drde}, ordered by trajectory difficulty, defined as the test RMSE of a deterministic SHRED reconstruction of the same trajectory. The trend is consistent: the interval score increases up to twice as steeply with trajectory difficulty for the baseline methods as for UQ-SHRED, indicating that UQ-SHRED degrades more mildly under distributional shift. (Fitting the interval score against difficulty separately for each initialisation gives a slope of $5.1 \pm 1.8$ for UQ-SHRED, against $6.9 \pm 1.2$ for the deep ensemble and up to $10.2 \pm 2.7$ for the heteroscedastic Gaussian.) UQ-SHRED obtains its intervals from the learned conditional distribution and therefore carries no scale fitted to the training trajectory, whereas methods that calibrate on a held-out set drawn from that trajectory, or that represent the predictive spread parametrically, transfer that calibration to shifted data incorrectly. The conformalised variant of UQ-SHRED belongs to the former category: its correction is calibrated on the ground-truth validation split, so the coverage it confers is guaranteed only under exchangeability between that split and the test data. On Run~1 it attains $6.9$ percentage points against the $4.0$ of the uncorrected model, reversing the advantage it held on Run~0.

\begin{figure}[t]
\centering
\includegraphics[width=\columnwidth]{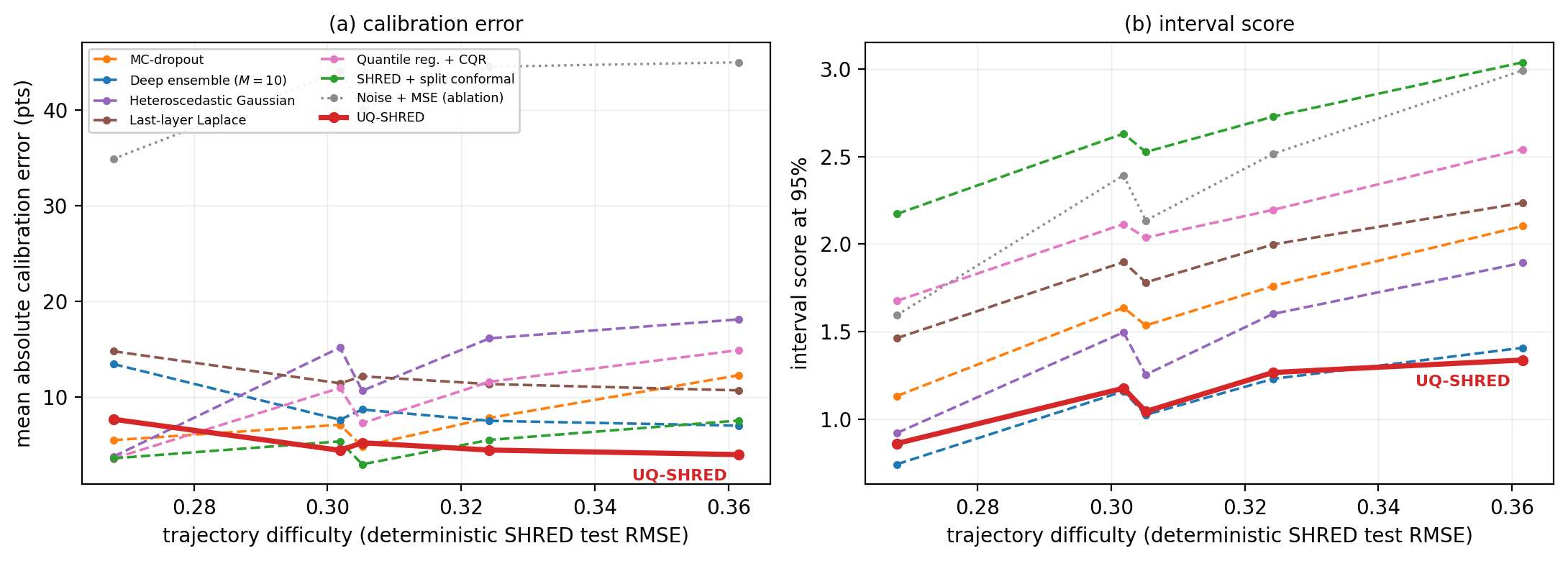}
\caption{Degradation across all five held-out trajectories, ordered by difficulty, measured as the test RMSE of a deterministic SHRED reconstruction of the same trajectory. (a) Calibration error; (b) interval score at the $95\%$ level. Lower is better in both panels. Means over ten random initialisations.}
\label{fig:baselines_rde}
\end{figure}

\paragraph{Contribution of the energy score}
The energy score, rather than the stochastic input, is what produces the predictive distribution. Trained under a mean-squared-error loss, the identical architecture with identical noise injection reconstructs the field as accurately as deterministic SHRED (RMSE $0.150$ against $0.149$) but attains only $56.5\%$ observed coverage at the $95\%$ nominal level, whereas training under the energy score raises observed coverage to $96.0\%$. The same contrast holds on Run~1, where the mean-squared-error loss variant attains $44.5\%$ observed coverage at the $95\%$ level against $91.0\%$ for UQ-SHRED.

\paragraph{Computational cost}
Table~\ref{tab:baselines_overhead} gives the parameter overhead each uncertainty model introduces. Concatenating a noise vector to the sensor input costs $4 d_h d_\varepsilon$ weights in the temporal unit, a quantity that does not depend on the state dimension $m$, whereas heteroscedastic Gaussian and quantile regression models widen the decoder output in proportion to $m$ and an ensemble multiplies every parameter by the ensemble size. Varying the spatial downsampling of the RDE fields over $m \in \{128,\dots,4800\}$ confirms the predicted behaviour: UQ-SHRED adds $12{,}800$ parameters at every resolution, while the heteroscedastic Gaussian grows from $51{,}328$ to $1{,}924{,}800$ and quantile regression from $513{,}280$ to $19{,}248{,}000$ (Table~\ref{tab:baselines_scaling}). The saving is in memory rather than in computation: the energy score requires two forward passes per optimisation step, and Monte Carlo inference requires $K$ passes where a parametric predictive distribution requires one. Training and inference times are reported in Appendix~\ref{app:baselines}.

\begin{table}[t]
\centering
\caption{Parameters introduced by each uncertainty model, relative to a deterministic SHRED network reconstructing a state of dimension $m$. Here $d_h$ denotes the hidden width of the temporal unit, $d_\varepsilon$ the noise dimension, $l_2$ the width of the final decoder layer, $Q$ the number of quantile levels and $M$ the ensemble size. UQ-SHRED is a method whose overhead does not grow with $m$; evaluated values are given in the final column of Table~\ref{tab:baselines_main}.}
\label{tab:baselines_overhead}
\footnotesize
\renewcommand{\arraystretch}{1.25}
\begin{tabular}{ll}
\toprule
\textbf{Method} & \textbf{Additional parameters} \\
\midrule
MC-dropout; Laplace; split conformal & $0$ \\
\textbf{UQ-SHRED} & $4 d_h d_\varepsilon$ \quad \emph{(independent of $m$)} \\
Heteroscedastic Gaussian & $l_2\, m$ \\
Quantile regression & $(Q-1)\, l_2\, m$ \\
Deep ensemble & $\approx M \times$ (all parameters) \\
\bottomrule
\end{tabular}
\end{table}

\paragraph{Scope}
The 1D RDE experiment provides a representative setting in which robustness to distributional shift can be assessed directly. The results support the claim that a single network trained under a proper scoring rule achieves calibration competitive with substantially more expensive alternatives, degrades more gracefully under distributional shift, and incurs an overhead that is independent of the dimension of the reconstructed field.

\begin{figure*}[!t]
\centering
\begin{subfigure}[b]{\columnwidth}
\centering
    \includegraphics[width=0.9\columnwidth, height=3.8cm, keepaspectratio]{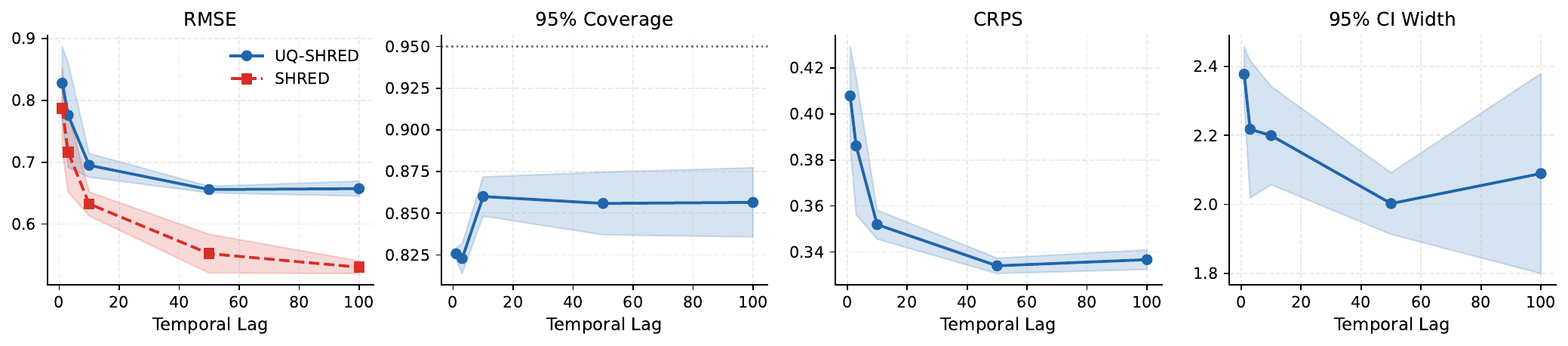}
    \caption{Effect of temporal lag on UQ-SHRED.}
    \label{fig:ablation_temporal}
\end{subfigure}
\begin{subfigure}[b]{\columnwidth}
\centering
    \includegraphics[width=0.9\columnwidth, height=3.8cm, keepaspectratio]{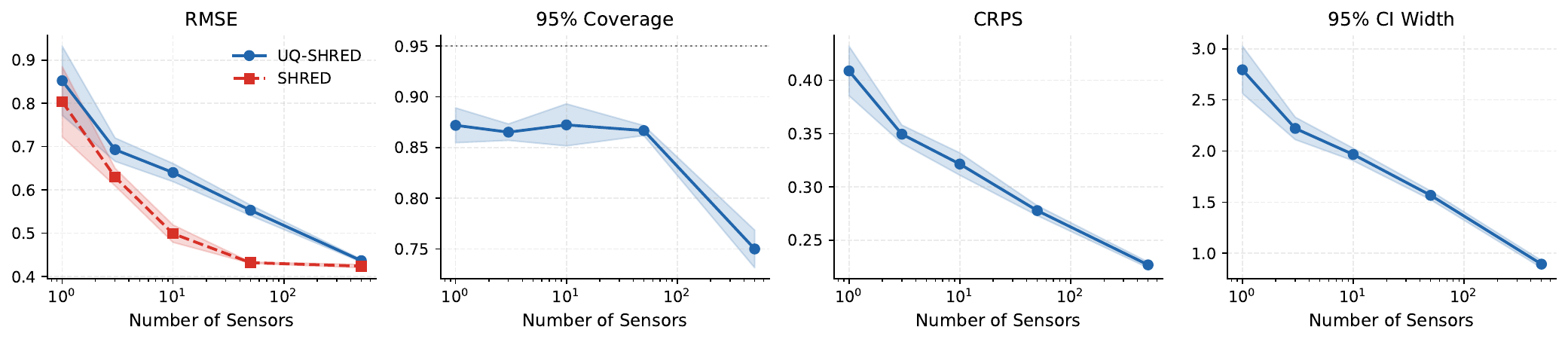}
    \caption{Effect of number of sensors on UQ-SHRED.}
    \label{fig:ablation_num_sensors}
\end{subfigure}
\begin{subfigure}[b]{\columnwidth}
\centering
    \includegraphics[width=0.9\columnwidth, height=3.8cm, keepaspectratio]{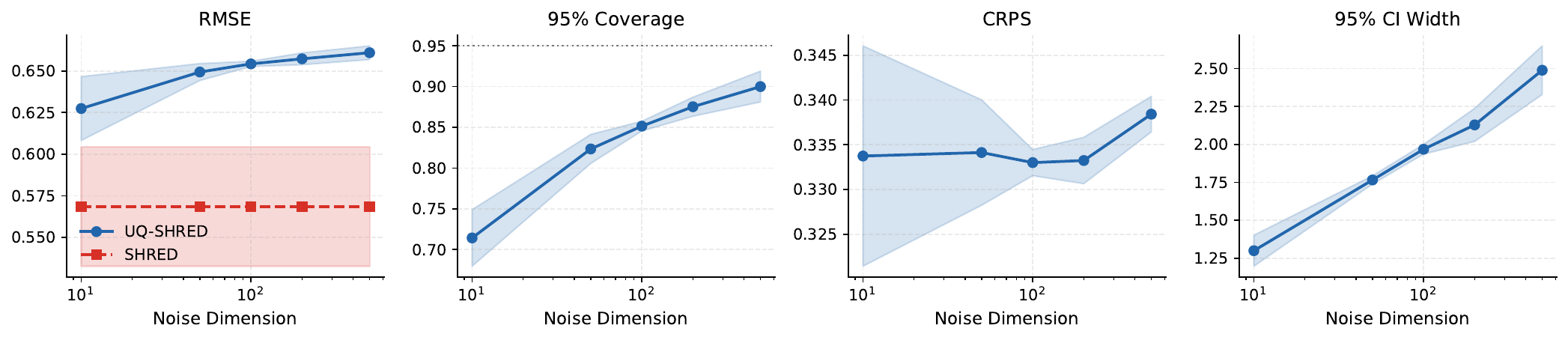}
    \caption{Effect of the dimension of $d_{\varepsilon}$ on UQ-SHRED.}
    \label{fig:ablation_noise_dim}
\end{subfigure}
\begin{subfigure}[b]{\columnwidth}
\centering
    \includegraphics[width=0.9\columnwidth, height=3.8cm, keepaspectratio]{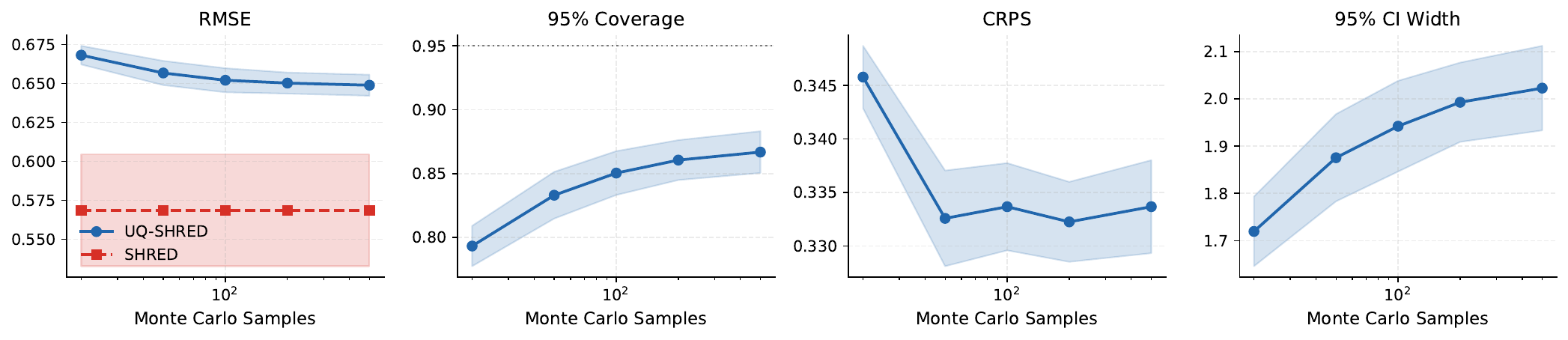}
    \caption{Effect of Monte Carlo sample size on UQ-SHRED.}
    \label{fig:ablation_n_samples}
\end{subfigure}
\begin{subfigure}[b]{\columnwidth}
\centering
    \includegraphics[width=0.9\columnwidth, height=3.8cm, keepaspectratio]{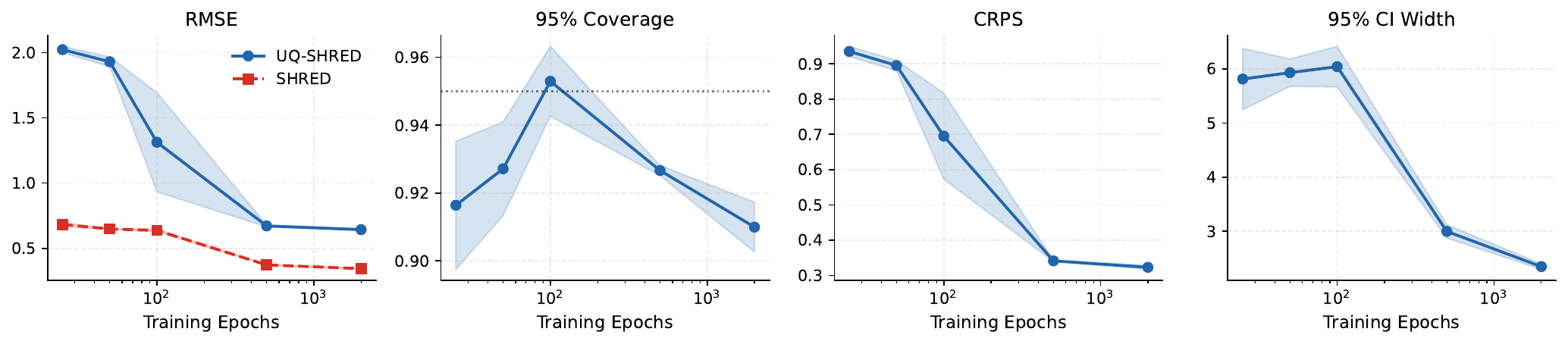}
    \caption{Effect of the number of training epochs on UQ-SHRED.}
    \label{fig:ablation_num_epochs}
\end{subfigure}
\end{figure*}

\subsection{Ablation study on SST data}
\label{sec:ablation}

As is common in evaluating neural network architectures, ablation studies are important in order to help understand the contribution of individual components to the model's overall performance. The ablation study performed here is critical for building interpretability and trust in the UQ-SHRED architecture while  allowing for more efficient model design by revealing what can be simplified without sacrificing accuracy.

\paragraph{Temporal lag parameter}

In Figure~\ref{fig:ablation_temporal}, we demonstrate how the temporal lag parameter will affect UQ-SHRED training and uncertainty quantification. Takens' embedding theorem guarantees that a sufficiently long temporal history provides a faithful embedding of the underlying attractor. Consistent with this, both SHRED and UQ-SHRED exhibit clear RMSE reduction as the lag increases, confirming the importance of temporal history for sparse sensing. We further observe that uncertainty coverage and confidence interval width are relatively insensitive to the temporal lag, particularly beyond $L=10$.

\paragraph{Sensor number}

In Figure~\ref{fig:ablation_num_sensors}, we demonstrate how the number of sensors will affect UQ-SHRED's performance. For sparse sensing problems, it is clear that with more sensors, the spatial prediction problem will be less challenging. 
This is supported by the numerical experiments that the testing loss is significantly reduced with more sensors. However, it is also important to notice from the $95\%$ coverage plot that the coverage performance is not improving much with additional sensors; it is likely that we need to increase the noise dimension $d_{\varepsilon}$ for larger input dimensions to address the uncertainty in a higher-dimensional space.

\paragraph{Noise dimension}

In Figure~\ref{fig:ablation_noise_dim}, we show how noise dimension can affect UQ-SHRED coverages. In general, a larger noise dimension helps the model better capture the uncertainty. 
In practice, we suggest the users to specify the noise dimension $d_{\varepsilon}$ scales with the size of input data. Having a significantly large noise dimension could help, but will also significantly increase the computational burden. 

\paragraph{Monte Carlo sample size}

In Figure~\ref{fig:ablation_n_samples}, we show how Monte Carlo sample size during inference can affect UQ-SHRED coverages. In theory, having more Monte Carlo samples will reduce the bias of the estimate of the quantiles, which is empirically verified here in the $95\%$ coverage plot. 
A common guideline is to have at least $100$ Monte Carlo samples for valid quantile estimation.

\paragraph{Epochs}

In Figure~\ref{fig:ablation_num_epochs}, we show an important observation that UQ-SHRED is able to quantify uncertainty for both overfitted and underfitted models. 
When the number of training epochs are small, the model is not learning enough information of the system. UQ-SHRED is still able to provide statistically valid uncertainty estimates which has an average of $91.8\%$ empirical coverage for $95\%$ quantile. When the model is well-trained, it is still able to provide a valid estimate with an average of $91.6\%$. 
Figure~\ref{fig:ablation_num_epochs} also suggests that applying early stopping methods can be a helpful technique for uncertainty quantification, which essentially avoids the model to overfit onto the noise.

\section{Conclusion}
\label{sec:conclusion}

In this paper, we introduce UQ-SHRED, a distributional learning framework for sparse sensing problems that enables valid uncertainty estimation. This stochastic modeling of uncertainty amends the deterministic reconstruction framework for regimes where the system is partially observed, stochastic, and/or contains high-frequency dynamics that are not captured by the available sensor data.
Inspired by engression~\cite{shen2023engression}, we concatenate a resampled Gaussian noise vector to the sensor input, and incorporate the energy score loss during training. The resulting model retains the computational efficiency while producing a distributional estimate at each spatiotemporal location through Monte Carlo resampling. Because the energy score is a strictly proper scoring rule, the population optimum is attained when the predictive distribution matches the true conditional distribution under technical assumptions in Thm.~\ref{thm:uq_shred_population}, and the quantile estimation converges to the true quantiles with large Monte Carlo samples (as demonstrated in Thm.~\ref{thm:monte_carlo}). 

The experimental study spans five datasets drawn from distinct scientific domains including climate science, computational fluid dynamics, neuroscience, astrophysics, and propulsion physics. These systems differ substantially in their spatiotemporal complexity, dimensionality, noise characteristics, and the degree to which sparse sensors constrain the full-field reconstruction. Across these diverse settings, UQ-SHRED consistently produces well-calibrated confidence intervals while maintaining reconstruction accuracy comparable to the deterministic SHRED baseline. The calibration diagrams show that the observed coverages closely follow the diagonal across all target confidence levels. More importantly, the confidence intervals exhibit spatially and temporally varying width corresponding to the true underlying complexity of the physical interactions. Specifically, the uncertainty estimates typically widen in regions of higher reconstruction ambiguity such as detonation wave fronts in the RDE, turning points in the SST seasonal cycle, and rapid transients in neural recordings. 
The uncertainty estimation tends to be narrow where the sensor data sufficiently constrains the reconstruction. 
These observations all indicate that the learned predictive distribution faithfully reflects the information provided by the sparse sensor measurements.

The ablation study on SST data (Section~\ref{sec:ablation}) reveals further insights in practice. Temporal lag primarily affects reconstruction accuracy through Takens's embedding theorem but has limited influence on calibration once a minimal embedding dimension is reached. The noise dimension $d_\varepsilon$ should scale with the input dimensionality to ensure sufficient expressive capacity for the stochastic mapping; a very small noise dimension will lead to underdispersed predictions, while excessively large dimensions increase computational cost without proportional benefit.  Most notably, UQ-SHRED maintains valid uncertainty estimates across a wide range of training epochs, including both undertrained and overtrained regimes, suggesting that the energy score loss provides robustness to model misspecification that could be absent from standard point-estimate training.

We note several limitations of the current framework, which can be improved by future works. First, the calibration guarantees of the energy score hold at the population optimum; with finite training data, the learned distribution may exhibit conservative or overconfident behavior depending on the dataset size and complexity, as observed in the RDE experiment where training on a single trajectory could lead to mild miscalibration under strong  distributional shift. Additionally, the current framework does not provide coverage guarantees for finite samples, which can be corrected by conformal calibration methods~\cite{lei2021conformal,barber2023conformal,huang2023uncertainty,gibbs2021adaptive}.

% End section statements

% AI use, \aucontribute{...}, \competing{...}:
% https://royalsociety.org/journals/authors/author-guidelines/ - "You are no longer required to add these statements in the manuscript itself - these statements will be automatically added to the paper if accepted for publication"

\dataccess{Our codebase is available at \url{https://github.com/gaoliyao/uq_shred}.}

\funding{The authors were supported in part by the US National Science Foundation (NSF) AI Institute for Dynamical Systems (dynamicsai.org), grant 2112085. JNK further acknowledges support from the Air Force Office of Scientific Research (FA9550-24-1-0141).}

\appendix

\begin{table}[t]
\centering
\caption{Uncertainty quantification methods on the 1D RDE transient stage, closer test trajectory (Run~0). Calibration reports observed coverage (\%) at each nominal level. Values are means over ten random initialisations.}
\label{tab:baselines_run0}
\footnotesize
\setlength{\tabcolsep}{3.2pt}
\renewcommand{\arraystretch}{1.15}
\resizebox{\textwidth}{!}{% scaled to fit the Proc. R. Soc. A text width
\begin{tabular}{lccccccccccc}
\toprule
 & \multicolumn{2}{c}{\textbf{RMSE}} & \multicolumn{5}{c}{\textbf{Calibration \%}} & & & & \\
\cmidrule(lr){2-3} \cmidrule(lr){4-8}
\textbf{Method} & $P$ & $T$ & 50 & 70 & 90 & 95 & 99 & \textbf{CRPS} & \textbf{IS}$_{95}$ & \textbf{Cal.\ err.} & \textbf{Nets} \\
\midrule
SHRED (deterministic)      & 0.149 & 0.274 &  0.0 &  0.0 &  0.0 &  0.0 &  0.0 & 0.0737 & --- & --- & 1 \\
MC-dropout                 & 0.149 & 0.283 & 61.6 & 79.3 & 91.0 & 93.2 & 95.2 & 0.0579 & 1.133 & 5.5 & 1 \\
Deep ensemble ($M{=}10$)   & 0.151 & 0.279 & 85.3 & 92.3 & 96.5 & 97.4 & 98.4 & 0.0496 & \textbf{0.744} & 13.4 & 10 \\
Heteroscedastic Gaussian   & 0.161 & 0.305 & 49.1 & 67.4 & 84.9 & 89.5 & 94.3 & 0.0622 & 0.922 & 3.8 & 1 \\
Last-layer Laplace         & 0.149 & 0.274 & 89.7 & 94.2 & 96.7 & 97.3 & 98.0 & 0.0823 & 1.463 & 14.8 & 1 \\
Quantile regression + CQR  & 0.159 & 0.290 & 55.9 & 74.6 & 93.2 & 97.1 & 97.0 & \textbf{0.0447} & 1.676 & 3.6 & 1 \\
SHRED + split conformal    & 0.148 & 0.273 & 56.5 & 74.9 & 92.6 & 96.5 & 96.5 & 0.0628 & 2.172 &  3.6 & 1 \\
\midrule
Noise injection + MSE \emph{(ablation)} & 0.150 & 0.283 & 22.7 & 33.9 & 49.7 & 56.5 & 66.8 & 0.0669 & 1.594 & 34.9 & 1 \\
\textbf{UQ-SHRED}          & 0.157 & 0.301 & 68.8 & 83.7 & 93.8 & 96.0 & 98.0 & 0.0677 & 0.862 & 7.7 & 1 \\
\textbf{UQ-SHRED + conformal} & 0.157 & 0.301 & 52.9 & 72.0 & 91.0 & 95.7 & 95.6 & 0.0677 & 1.223 & \textbf{2.0} & 1 \\
\bottomrule
\end{tabular}}
\end{table}

\begin{table}[t]
\centering
\caption{Results for the further test trajectory (Run~1), with metrics and reporting conventions as in Table~\ref{tab:baselines_run0}.}
\label{tab:baselines_run1}
\footnotesize
\setlength{\tabcolsep}{3.2pt}
\renewcommand{\arraystretch}{1.15}
\resizebox{\textwidth}{!}{% scaled to fit the Proc. R. Soc. A text width
\begin{tabular}{lccccccccccc}
\toprule
 & \multicolumn{2}{c}{\textbf{RMSE}} & \multicolumn{5}{c}{\textbf{Calibration \%}} & & & & \\
\cmidrule(lr){2-3} \cmidrule(lr){4-8}
\textbf{Method} & $P$ & $T$ & 50 & 70 & 90 & 95 & 99 & \textbf{CRPS} & \textbf{IS}$_{95}$ & \textbf{Cal.\ err.} & \textbf{Nets} \\
\midrule
SHRED (deterministic)      & 0.202 & 0.406 &  0.0 &  0.0 &  0.0 &  0.0 &  0.0 & 0.1219 & --- & --- & 1 \\
MC-dropout                 & 0.192 & 0.395 & 42.2 & 58.7 & 75.5 & 80.4 & 85.9 & 0.0933 & 2.102 & 12.3 & 1 \\
Deep ensemble ($M{=}10$)   & 0.181 & 0.373 & 62.2 & 74.7 & 85.5 & 88.5 & 91.9 & \textbf{0.0811} & 1.408 &  7.0 & 10 \\
Heteroscedastic Gaussian   & 0.188 & 0.391 & 36.5 & 51.5 & 68.7 & 74.6 & 82.1 & 0.0956 & 1.893 & 18.1 & 1 \\
Last-layer Laplace         & 0.202 & 0.406 & 78.9 & 86.2 & 91.1 & 92.5 & 94.3 & 0.1160 & 2.235 & 10.7 & 1 \\
Quantile regression + CQR  & 0.186 & 0.405 & 32.9 & 48.9 & 75.7 & 86.5 & 85.5 & 0.0884 & 2.541 & 14.9 & 1 \\
SHRED + split conformal    & 0.202 & 0.406 & 41.8 & 59.5 & 83.3 & 90.9 & 90.9 & 0.1065 & 3.037 & 7.5 & 1 \\
\midrule
Noise injection + MSE \emph{(ablation)} & 0.197 & 0.403 & 16.8 & 25.3 & 38.4 & 44.5 & 54.3 & 0.1070 & 2.991 & 45.0 & 1 \\
\textbf{UQ-SHRED}          & 0.193 & 0.393 & 56.0 & 72.9 & 87.3 & 91.0 & 94.7 & 0.0943 & \textbf{1.338} & \textbf{4.0} & 1 \\
\textbf{UQ-SHRED + conformal} & 0.193 & 0.393 & 42.2 & 60.6 & 83.8 & 91.3 & 91.6 & 0.0943 & 1.670 &  6.9 & 1 \\
\bottomrule
\end{tabular}}
\end{table}

\section{Baseline comparison: protocol and full results}
\label{app:baselines}

\paragraph{Protocol}
All methods are built on the SHRED architecture of Sec.~\ref{sec:method}, using the same temporal unit and shallow decoder, optimiser settings, epoch budget, and early-stopping criterion; the methods differ only in the component through which uncertainty is represented, which may be the output layer, the training objective, or a post-hoc calibration applied to a trained network. Method-specific hyperparameters---the dropout rate and observation noise for MC-dropout, the prior precision and noise scale for the Laplace approximation, and the $\beta$ weighting of the heteroscedastic likelihood~\cite{seitzer2022pitfalls}---are each selected by grid search on validation continuous ranked probability score, while UQ-SHRED retains the configuration of Sec.~\ref{sec:uq_1drde}; a grid search over $d_\varepsilon$ under the same protocol alters its calibration error by at most $0.5$ percentage points and leaves the ordering of Table~\ref{tab:baselines_main} unaffected. Deep ensembles follow~\cite{lakshminarayanan2017simple}: $M = 10$ independently initialised networks, each predicting a mean and a variance, combined as a Gaussian mixture. Every method is evaluated with $K = 500$ Monte Carlo draws, larger than in Sec.~\ref{sec:uq_1drde} because empirical quantiles at the extreme levels are too narrow at small $K$, and parametric predictive distributions are sampled rather than inverted in closed form, so that all methods incur the same finite-sample quantile bias. Each configuration is trained once on the ground-truth trajectory and evaluated on all five held-out ensemble members. Split conformal and conformalised quantile regression~\cite{romano2019conformalized} are calibrated on the ground-truth validation split only, never on held-out ensemble members, since the latter carry the distributional shift under test. The conformalised variant of UQ-SHRED applies conformalised quantile regression to its Monte Carlo quantiles: the conformity score is the signed distance of the observation from the nearer endpoint of the central interval, and the calibrated correction widens both endpoints symmetrically; the procedure requires only a nested family of interval estimates and is therefore indifferent to how the conditional quantiles are obtained. All reported values are means over ten random initialisations, and RMSE is computed after interpolating the reconstructed and reference states to the 4800-point spatial grid.

\paragraph{Training and inference cost}
The overhead reported in Sec.~\ref{sec:baselines} is in parameters and memory; the corresponding computational cost is as follows, measured on the RDE data at $m = 128$ and $m = 4800$. On an Apple M4 CPU, training UQ-SHRED costs $7.0$ and $21.0$\,ms per epoch respectively, against $4.0$ and $12.2$\,ms for a deterministic SHRED network, the difference arising from the two forward passes the energy score requires at each optimisation step; among single-network methods this makes UQ-SHRED the most expensive at the smaller resolution and the least expensive at the larger, where the heteroscedastic Gaussian costs $21.6$\,ms per epoch. A ten-member ensemble costs $40.5$ and $217.0$\,ms per epoch at the two resolutions. At inference the ordering differs: methods with a parametric predictive distribution require one forward pass and an ensemble requires $M$, whereas UQ-SHRED and MC-dropout require $K$, giving $234$ and $382$\,ms for a batch of thirty-two against $0.5$ and $1.4$\,ms for the heteroscedastic Gaussian. The comparison therefore trades no additional parameters and no additional networks against sampling at inference. Unlike the ensemble size $M$, which is fixed once training begins, $K$ may be reduced after training, and the ablation of Sec.~\ref{sec:ablation} indicates that $K \approx 100$ is sufficient for the quantiles reported here.

\begin{table}[t]
\centering
\caption{Additional parameters relative to a deterministic SHRED model as a function of the reconstructed state dimension $m$, obtained by varying the spatial downsampling of the 1D RDE fields. The final column extrapolates to the sea-surface temperature field of Sec.~\ref{sec:experiments}.}
\label{tab:baselines_scaling}
\footnotesize
\setlength{\tabcolsep}{4pt}
\renewcommand{\arraystretch}{1.15}
\begin{tabular}{lrrrrr}
\toprule
\textbf{Method} & $m{=}128$ & $m{=}640$ & $m{=}1372$ & $m{=}4800$ & $m{=}44{,}219$ \\
\midrule
\textbf{UQ-SHRED}         & \textbf{12{,}800} & \textbf{12{,}800} & \textbf{12{,}800} & \textbf{12{,}800} & \textbf{12{,}800} \\
Heteroscedastic Gaussian  & 51{,}328 & 256{,}640 & 550{,}172 & 1{,}924{,}800 & 17{,}731{,}819 \\
Quantile regression       & 513{,}280 & 2{,}566{,}400 & 5{,}501{,}720 & 19{,}248{,}000 & 177{,}318{,}190 \\
Deep ensemble ($M{=}10$)  & 2{,}968{,}894 & 6{,}869{,}822 & 12{,}446{,}930 & 38{,}564{,}862 & 338{,}831{,}407 \\
\bottomrule
\end{tabular}
\end{table}

\section{Additional Figures}
\label{app:additional_figures}

\paragraph{Spatial reconstruction visualizations}
\label{app:spatial_figures}

Figures~\ref{fig:uq_sst_spacetime}--\ref{fig:uq_1drde_spacetime} present spatial reconstruction visualizations for the first four experimental datasets. Each figure displays, from left to right: the ground-truth field, the UQ-SHRED median reconstruction, the absolute error, and the $95\%$ CI width, at selected time steps. In all cases, the spatial distribution of the $95\%$ CI width closely tracks the spatial distribution of reconstruction error, indicating that UQ-SHRED assigns higher uncertainty where the reconstruction is less accurate.

\begin{figure}[t]
\centering
\includegraphics[width=\columnwidth]{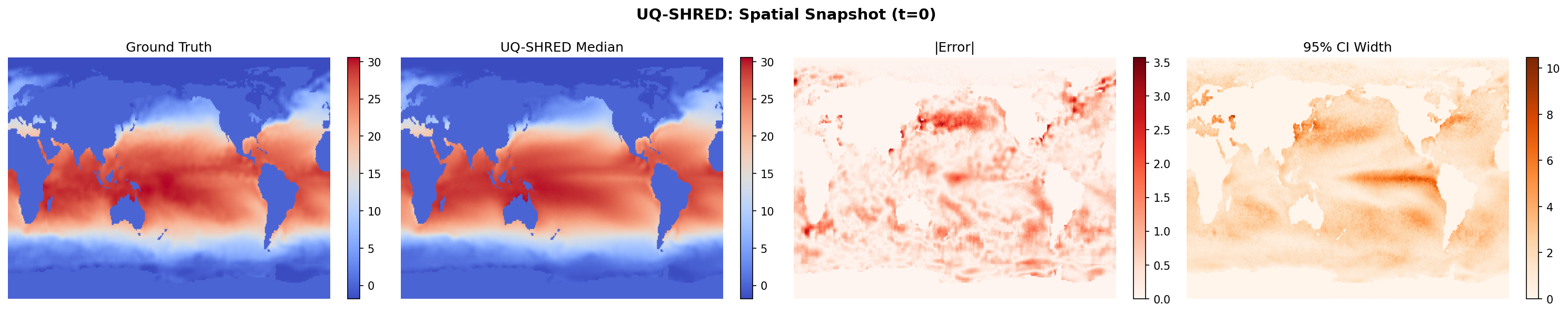}
\caption{Spatial reconstruction of UQ-SHRED on SST data. Columns: ground truth, median reconstruction, absolute error, and $95\%$ CI width. }
\label{fig:uq_sst_spacetime}
\end{figure}

\begin{figure}[t]
\centering
\includegraphics[width=0.7\columnwidth]{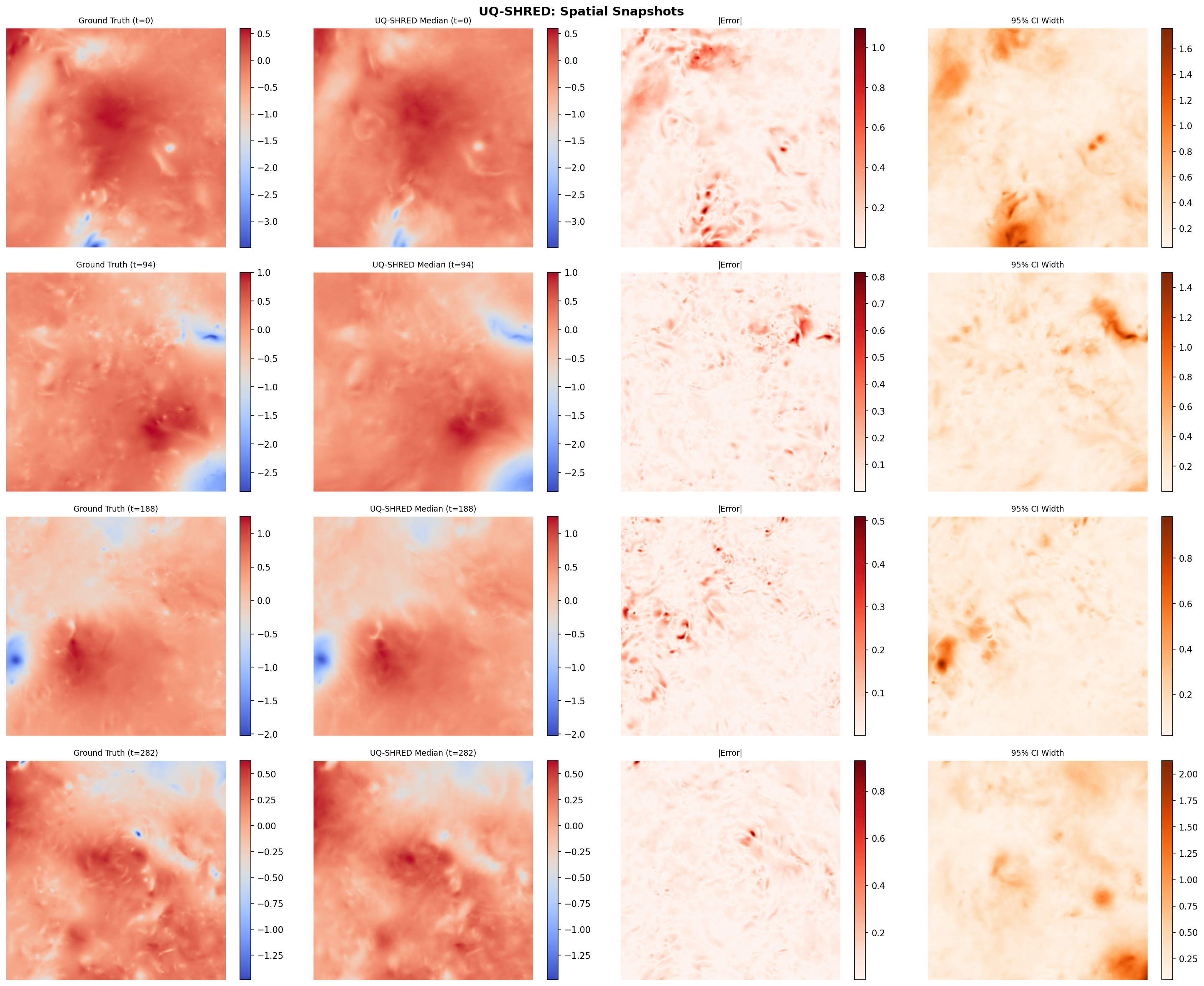}
\caption{Spatial reconstruction of UQ-SHRED on isotropic turbulent flow data. Columns: ground truth, median reconstruction, absolute error, and $95\%$ CI width. }
\label{fig:uq_iso_spatial}
\end{figure}

\begin{figure}[t]
\centering
\includegraphics[width=0.8\columnwidth]{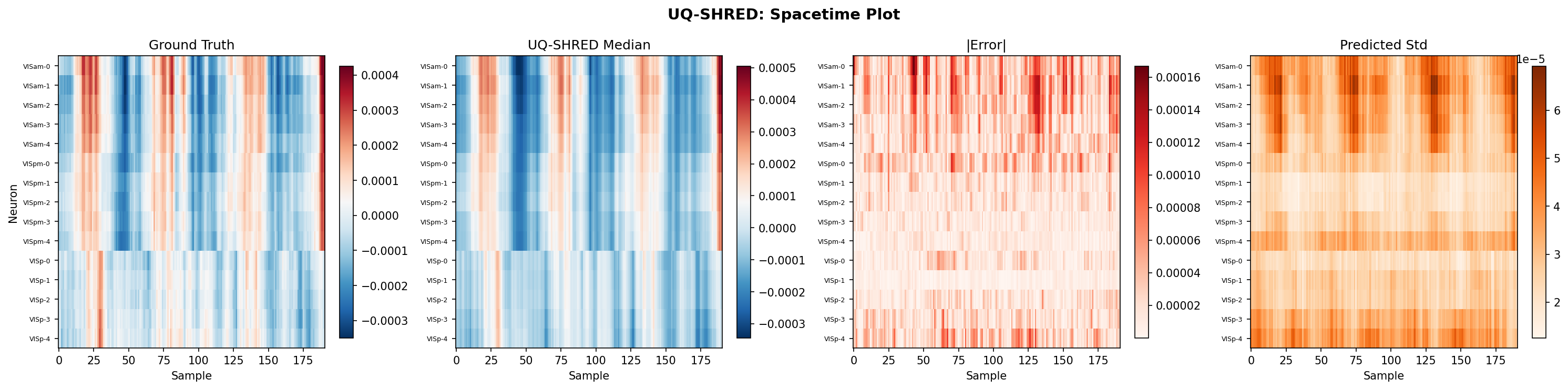}
\caption{Spatial reconstruction of UQ-SHRED on neuro activity data. Columns: ground truth, median reconstruction, absolute error, and $95\%$ CI width. }
\label{fig:uq_neuro_spacetime}
\end{figure}

\begin{figure}[t]
\centering
\includegraphics[width=0.65\columnwidth]{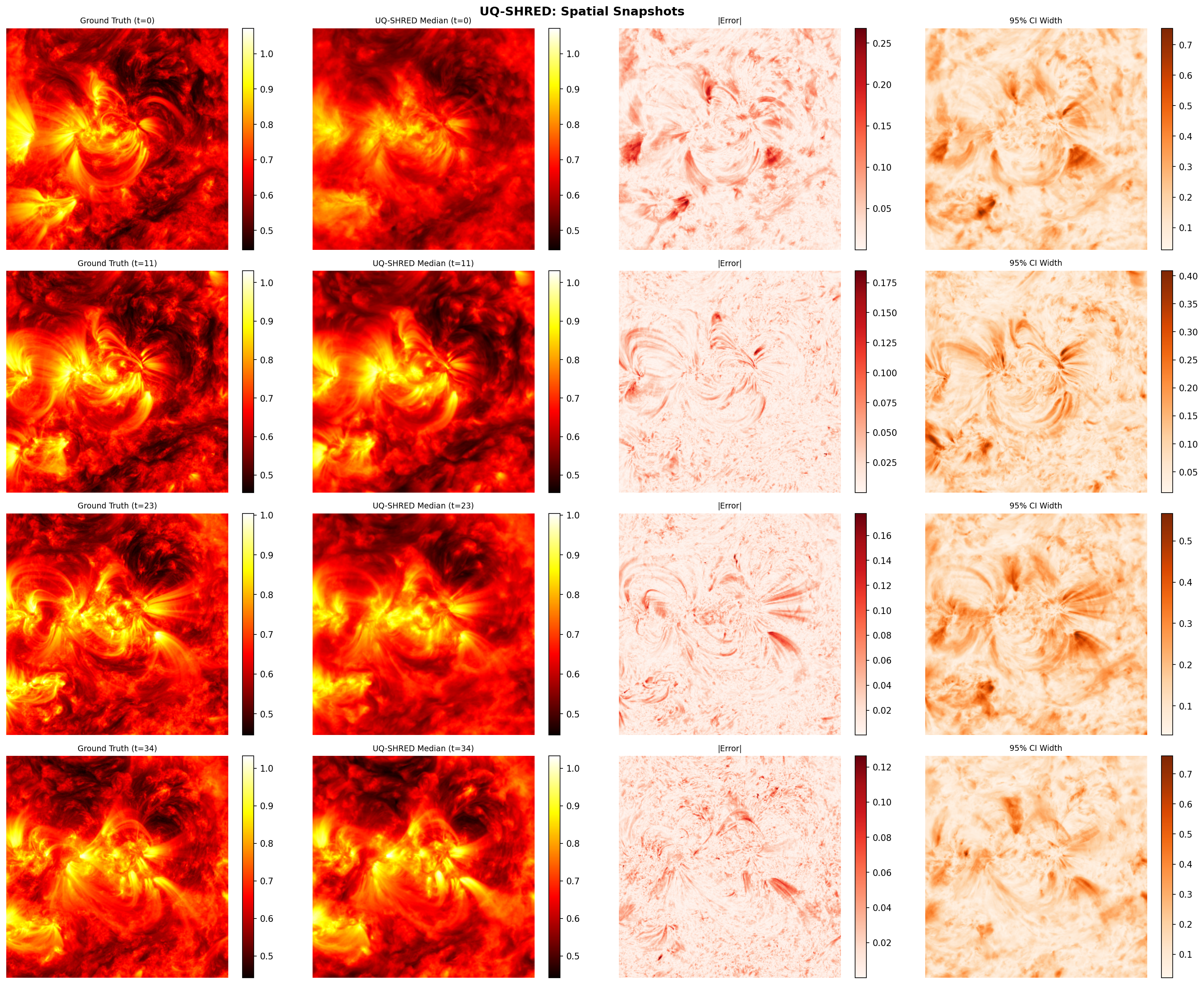}
\caption{Spatial reconstruction of UQ-SHRED on solar activity data. Columns: ground truth, median reconstruction, absolute error, and $95\%$ CI width. }
\label{fig:uq_sun_spatial}
\end{figure}

\vspace{-8mm}

\begin{figure}[t]
\centering
\begin{subfigure}[b]{0.7\columnwidth}
    \centering
    \includegraphics[width=\columnwidth]{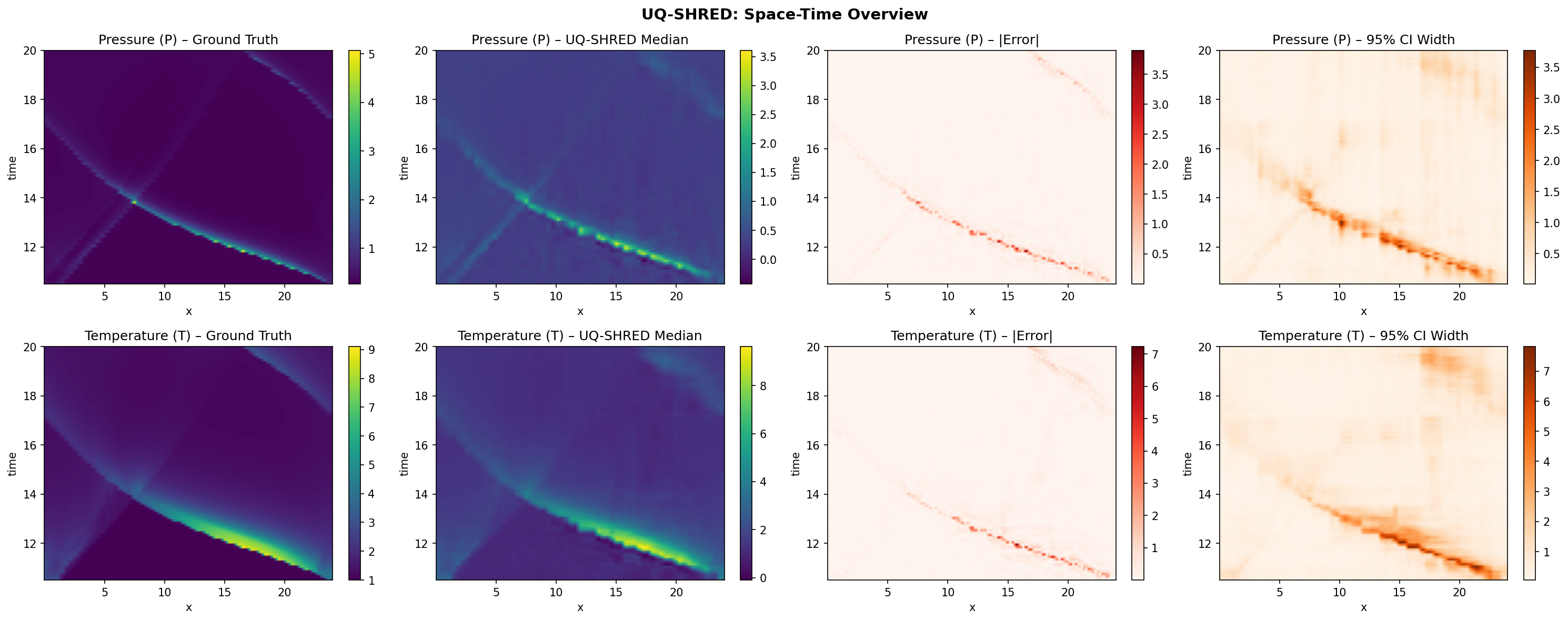}
    \caption{Closer trajectory (Run~0).}
\end{subfigure}
\vspace{0.3em}
\begin{subfigure}[b]{0.7\columnwidth}
    \centering
    \includegraphics[width=\columnwidth]{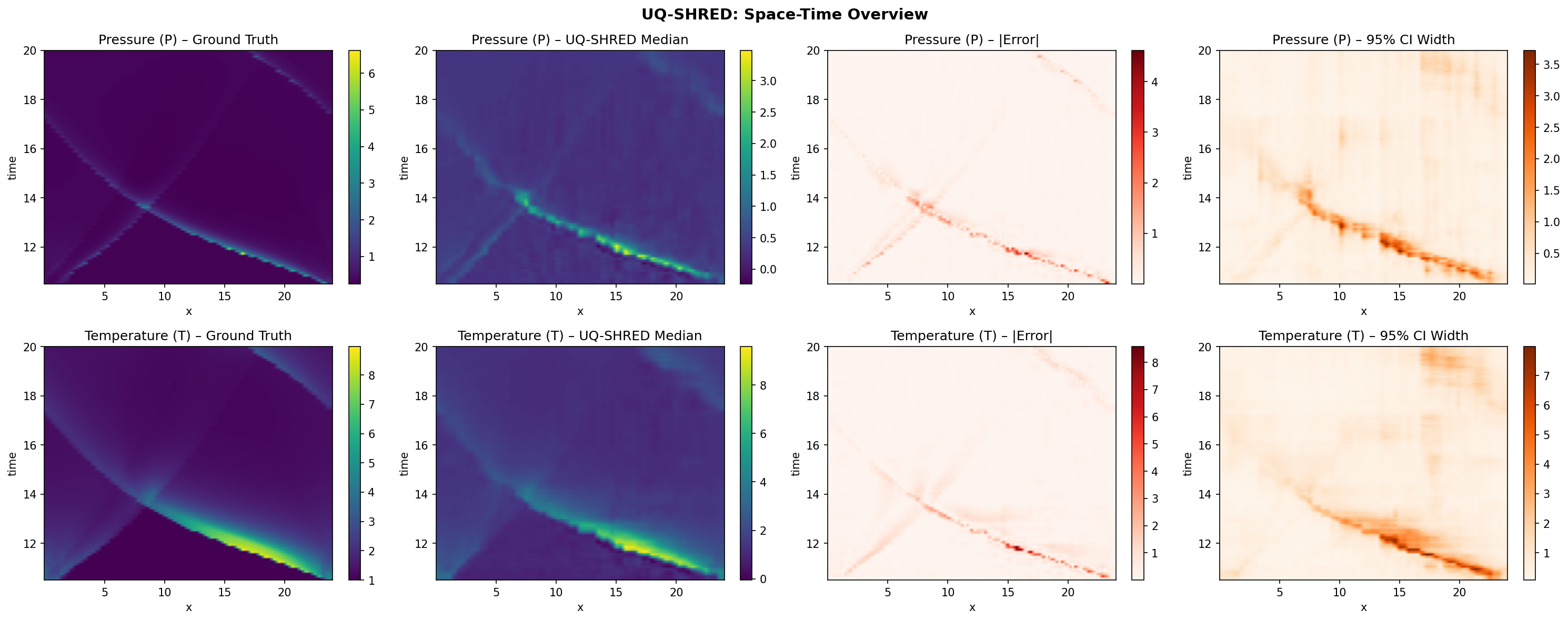}
    \caption{Further trajectory (Run~1).}
\end{subfigure}
\caption{Spacetime overview for UQ-SHRED on the 1D RDE transient ignition stage. Columns: ground truth, median reconstruction, absolute error, and $95\%$ CI width. Top rows: pressure; bottom rows: temperature. The $95\%$ CI width concentrates along the detonation fronts, matching the spatial distribution of reconstruction error.}
\label{fig:uq_1drde_spacetime}
\end{figure}

\newpage

\paragraph{1D RDE transient stage: additional figures}
\label{app:uq_1drde}

Figure~\ref{fig:uq_1drde_timeseries} presents temporal traces at selected spatial locations for both the closer and further RDE test trajectories. The confidence bands widen during detonation-wave passage and narrow during quiescent intervals, consistent with the spatial snapshots shown in Figure~\ref{fig:uq_1drde_snapshots}.

\begin{figure}[t]
\centering
\begin{subfigure}[b]{\columnwidth}
    \centering
    \includegraphics[width=\columnwidth]{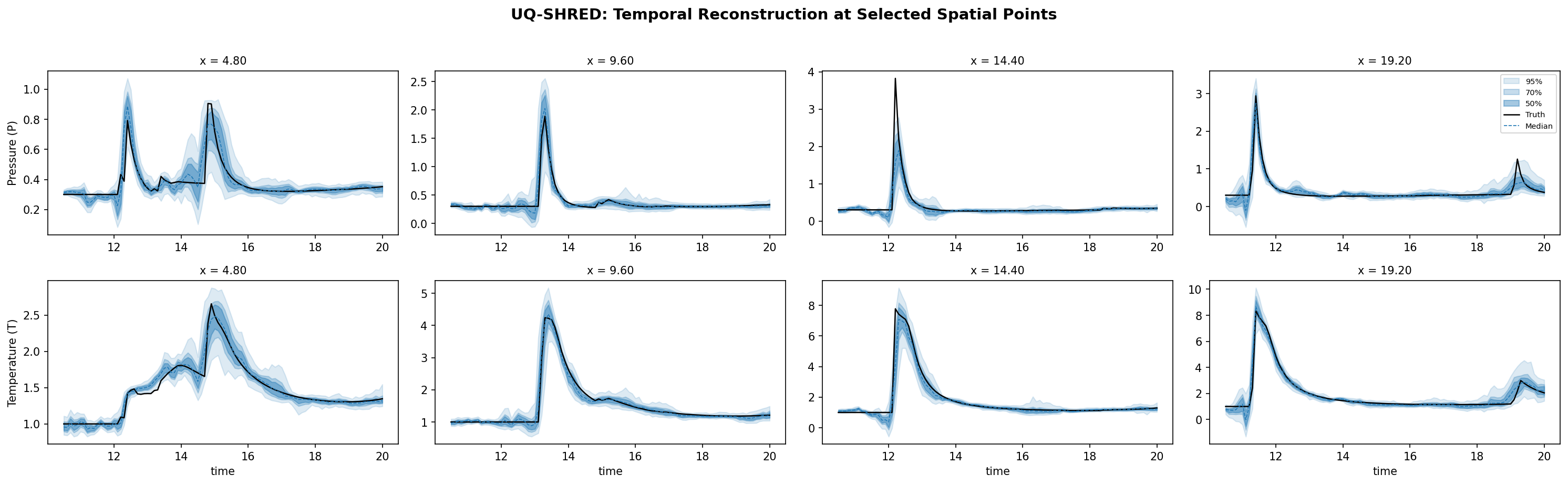}
    \caption{Closer trajectory (Run~0).}
\end{subfigure}
\vspace{0.3em}
\begin{subfigure}[b]{\columnwidth}
    \centering
    \includegraphics[width=\columnwidth]{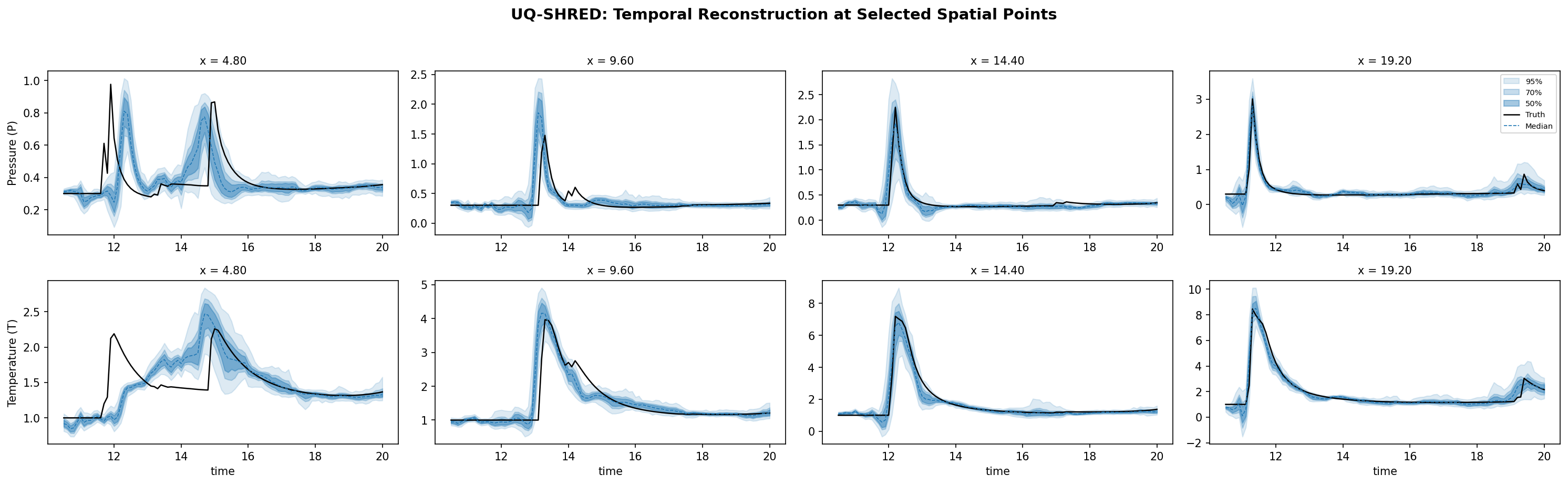}
    \caption{Further trajectory (Run~1).}
\end{subfigure}
\caption{Temporal traces at selected spatial locations for UQ-SHRED on the 1D RDE transient ignition stage. Confidence bands widen during detonation-wave passage and contract during quiescent intervals.}
\label{fig:uq_1drde_timeseries}
\end{figure}

\bibliographystyle{RS}
\bibliography{refs}

\end{document}